\documentclass{article}

\usepackage{microtype}
\usepackage{graphicx}
\usepackage{subfigure}
\usepackage{booktabs} 
\usepackage{multirow}
\usepackage{placeins}    

\usepackage{hyperref}

\usepackage{xspace}

\usepackage{amsmath}
\usepackage{enumitem}

\usepackage[accepted]{icml2025}

\usepackage{amsmath}
\usepackage{amssymb}
\usepackage{mathtools}
\usepackage{amsthm}

\usepackage{microtype}
\usepackage{etoolbox}
\usepackage{flushend}
\usepackage{mathtools}
\usepackage[capitalize,noabbrev]{cleveref}

\theoremstyle{plain}
\newtheorem{theorem}{Theorem}[section]

\theoremstyle{definition}
\newtheorem{definition}[theorem]{Definition}

\theoremstyle{remark}

\usepackage[textsize=tiny]{todonotes}

\icmltitlerunning{ProDiff: Prototype-Guided Diffusion for Minimal Information Trajectory Imputation}

\begin{document}

\twocolumn[
\icmltitle{ProDiff: Prototype-Guided Diffusion for Minimal Information Trajectory Imputation}


\icmlsetsymbol{equal}{*}

\begin{icmlauthorlist}
\icmlauthor{Tianci Bu}{equal,yyy}
\icmlauthor{Le Zhou}{equal,yyy}
\icmlauthor{Wenchuan Yang}{equal,yyy}
\icmlauthor{Jianhong Mou}{yyy}
\icmlauthor{Kang Yang}{yyy1}
\icmlauthor{Suoyi Tan}{yyy}
\icmlauthor{Feng Yao}{yyy}
\icmlauthor{Jingyuan Wang}{yyy2,yyy3,yyy4}
\icmlauthor{Xin Lu}{yyy}
\end{icmlauthorlist}

\icmlaffiliation{yyy}{College of Systems Engineering, National University of Defense Technology, Changsha, China}
\icmlaffiliation{yyy2}{School of Computer Science and Engineering, Beihang University, Beijing, China}
\icmlaffiliation{yyy1}{School of Information, Renmin University of China, Beijing, China}
\icmlaffiliation{yyy3}{School of Economics and Management, Beihang University, Beijing 100191, China}
\icmlaffiliation{yyy4}{Engineering Research Center of Advanced Computer Application Technology, Ministry of Education}

\icmlcorrespondingauthor{Jingyuan Wang}{jywang@buaa.edu.cn}
\icmlcorrespondingauthor{Xin Lu}{xin.lu.lab@outlook.com}

\icmlkeywords{Machine Learning, ICML}

\vskip 0.3in
]



\printAffiliationsAndNotice{\icmlEqualContribution}  

\begin{abstract}
Trajectory data is crucial for various applications but often suffers from incompleteness due to device limitations and diverse collection scenarios. Existing imputation methods rely on sparse trajectory or travel information, such as velocity, to infer missing points. However, these approaches assume that sparse trajectories retain essential behavioral patterns, which place significant demands on data acquisition and overlook the potential of large-scale human trajectory embeddings.
To address this, we propose ProDiff, a trajectory imputation framework that uses only two endpoints as minimal information. It integrates prototype learning to embed human movement patterns and a denoising diffusion probabilistic model for robust spatiotemporal reconstruction. Joint training with a tailored loss function ensures effective imputation.
ProDiff outperforms state-of-the-art methods, improving accuracy by 6.28\% on FourSquare and 2.52\% on WuXi. Further analysis shows a 0.927 correlation between generated and real trajectories, demonstrating the effectiveness of our approach.
\end{abstract}

\section{Introduction}
\label{submission}
Mining spatio-temporal patterns from trajectory data has broad applications, such as infectious diseases control, human behavioral analysis, and urban planning \cite{jia2020population,  zhang2023heterogeneous, zheng2014urban, hettigeairphynet, ji2022precision}. Such data primarily originates from Location-Based Services (LBS) using cell tower signals \cite{lu2012predictability}, satellite-based systems such as GPS, GLONASS, BeiDou, QZSS, and Galileo, as well as IP-based location methods utilized by online platforms.

\begin{figure}[t]
\centering
\includegraphics[width=1.0\linewidth]{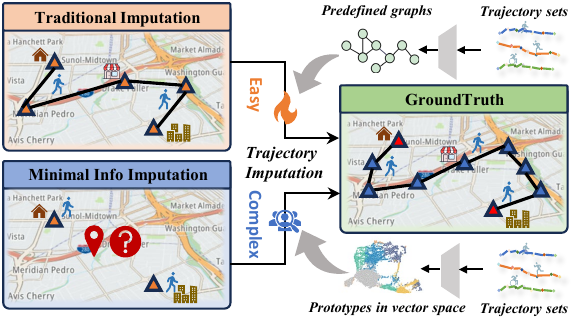}
\caption{ Comparison of traditional and proposed trajectory imputation. Traditional methods preserve movement patterns but impose device constraints and rely on predefined graphs. Our approach directly embeds trajectories into vector space for minimal information imputation. }
\label{Intro}
\end{figure}


Most of the trajectory mining tasks \cite{bao2021bilstm, yao2017serm, wang2018learning} and methods\cite{li2018multi, huang2022context, shen2020ttpnet} are based on the assumption of complete and accurate trajectory data \cite{chen2024deep}, making them sensitive to the granularity and accuracy of sampled data. 
However, contemporary location data collection, reliant on mobile networks or satellite communications, is often hindered by base station coverage gaps, signal instability, and environmental interference, leading to frequent missing data. Traditional methods like linear interpolation \cite{blu2004linear} and vector autoregressive models \cite{lutkepohl2013vector} provide efficient solutions but often fail to capture the full data distribution. Deep learning-based imputation methods like \citet{wu2023timesnet, du2023saits, xia2021attnmove} capture spatial-temporal dependencies using self-attention mechanisms or convolutional neural network, while some methods based on graph neural networks \cite{chen2023rntrajrec, wei2024micro} rely on predefined graph structures to extract spatial-temporal features. Recently, generative methods such as generative adversarial networks (GANs) \cite{jiang2023continuous} and variational autoencoders (VAEs) \cite{chen2021trajvae} have shown promise in trajectory synthesis, and the emergence of denoising diffusion probabilistic models (DDPMs) \cite{ho2020denoising} has further advanced the field. For instance, DiffTraj \cite{10.5555/3666122.3668965} leverages diffusion model to capture group-level trajectories, generating synthetic trajectory data while preserving privacy.

Despite progress, existing trajectory imputation methods face notable limitations. First, they typically assume that sampled trajectories, despite large intervals, retain essential movement patterns, interpolating local points using global trajectories thereby imposing constraints on devices and operational environments. Second, they fail to fully embed the vast amount of unlabeled human trajectories, which exhibit consistent macro-level patterns that enable imputation under more relaxed conditions, as shown in Fig. \ref{Intro}. 
Although recent work \cite{wei2024micro} has leveraged unlabeled trajectories to aid imputation, it required the graph structures as the foundamental element for prediction.




To address these limitations, we introduce ProDiff, a framework integrating \underline{Pro}totype learning with a denoising \underline{Diff}usion probabilistic model. ProDiff operates under minimal information constraints, modeling a trajectory as a sequence of points and interpolating missing locations using only two endpoints within a fixed-length window. This relaxes the prior assumption that sparse trajectories must retain essential movement patterns. ProDiff consists of two key components: (i) Diffusion-based generative model: The diffusion model reconstructs human movement by iteratively denoising from a latent space, offering reliable spatiotemporal modeling. (ii) Prototype-based condition extractor: This module learns prototypes that represent individual movement patterns, embedding diverse trajectories into a vector space through self-supervised learning. Given known trajectory information as queries, it extracts a comprehensive pattern representation to guide the diffusion model in generating realistic, individualized trajectories.
To effectively couple these two components, we design a joint training loss function that integrates generative and prototype learning objectives. This ensures a more compact embedding space while mitigating independent error accumulation and irreversible information loss typically introduced by multi-stage training. 
Experimental results demonstrate the effectiveness of the proposed ProDiff and prove that the captured underlying trajectory structures can signficantly improve the imputation accuracy.

In summary, the contributions of this work are as follows:

\begin{itemize}
    \item We relax the prior assumption that sparse trajectories inherently retain movement patterns and introduce trajectory imputation under minimal information constraints.
    \item We propose a prototype-based condition extractor that embeds human trajectories into a vector space, capturing macro-level behavioral patterns for the first time in trajectory imputation.
    \item We develop ProDiff, a framework that jointly optimizes generative modeling and prototype learning, effectively reconstructing missing trajectory data while reducing independent errors.
    \item We conduct extensive experiments on WuXi \cite{song2017recovering} and Foursquare \cite{yang2014modeling}, demonstrating superior imputation accuracy across different trajectory window sizes. Our code is available at 
    \url{https://github.com/b010001y/ProDiff}.
\end{itemize}

\section{Related Work} \label{sec:main:relw}

Please refer to \Cref{sec:appendix:relw} for an extensive discussion of related work. Here we provide its summary.

\textbf{Spatial-Temporal Sequence Imputation.} 
Traditional imputation methods evolved from simple statistical approaches like linear interpolation \cite{blu2004linear} to probabilistic frameworks such as PCA and Bayesian networks \cite{qu2009ppca,shi2013missing}, are fast but often too simple to capture complex distributions. Deep learning revolutionized the field through two paradigms: non-generative models like GRU-D \cite{che2018recurrent} with temporal decay mechanisms and SAITS \cite{du2023saits} using masked self-attention, and generative approaches where diffusion models like Diffusion-TS \cite{yuan2024diffusion} now dominate by disentangling trend-seasonality components.


\textbf{Trajectory Data Mining.}
Trajectory analysis spans forecasting, estimation, and anomaly detection. CNN/RNN architectures \cite{bao2021bilstm,yang2017neural} pioneered point-wise prediction, while road-aware models like WDR \cite{wang2018learning} advanced travel time estimation through road network embeddings. Anomaly detection evolved from RNN-based classifiers \cite{song2018anomalous} to latent space methods like GM-VSAE \cite{liu2020online}. The emerging mobility generation field bridges forecasting and synthesis, exemplified by DiffTraj \cite{10.5555/3666122.3668965} applying raw GPS diffusion, yet lacks physics-aware trajectory topology preservation – a gap our work addresses.


\begin{figure*}[ht]
\centering
\includegraphics[width=1.0\linewidth]{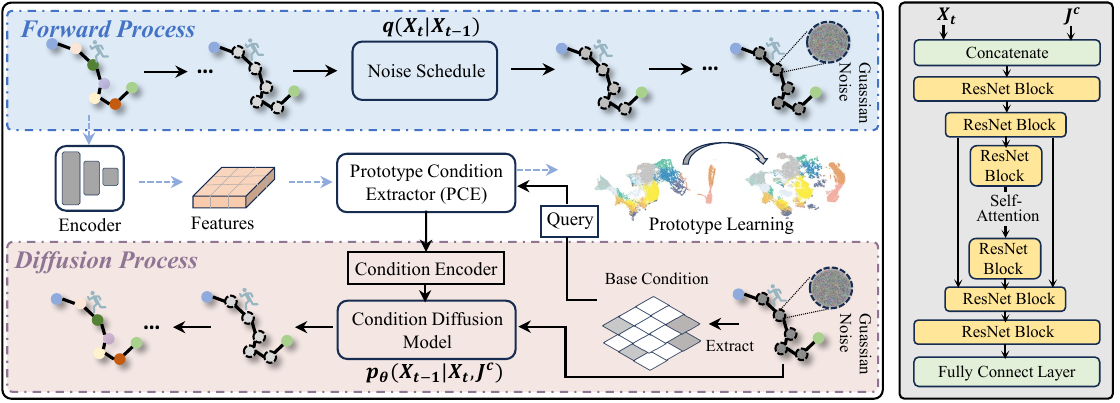}
\caption{Left illustrates how prototype learning and diffusion models interact. The diffusion process progressively corrupts trajectories with Gaussian noise, preserving only the endpoints, while prototype learning embeds trajectories and extracts patterns. During denoising, prototype-based conditions, combined with endpoint features, guide the diffusion model. A joint loss function optimizes both components, ensuring effective trajectory reconstruction. Right is the architecture of the diffusion base model.}
\label{frame}
\end{figure*}

\textbf{Mobility Data Synthesizing.}
Early synthesis relied on statistical approximations \cite{simini2021deep} until VAEs \cite{chen2021trajvae} and GANs \cite{jiang2023continuous} introduced deep generative modeling. Graph-based innovations like RNTrajRec \cite{chen2023rntrajrec} captured semantics through spatial-temporal transformers, while attention architectures \cite{xia2021attnmove} explicitly modeled cross-region dependencies. Modern diffusion frameworks such as ControlTraj \cite{zhu2024controltraj} enable conditional generation via traffic signal conditioning but remain resolution-rigid.

\section{ProDiff Model}

\subsection{Problem Definition}

\begin{definition}
\label{def:st}
\textbf{Spatio-Temporal Trajectory.}  
A spatio-temporal trajectory is a sequence of human activity points, denoted as $\mathbf{x}_{i,j} \in \mathbb{R}^{n}$, where $n$ represents the number of attributes. Each point consists of \textit{time}, \textit{longitude}, and \textit{latitude},  \emph{i.e.}, $\mathbf{x}_{i,j} = \{t_{i,j}, lon_{i,j}, lat_{i,j}\}$, satisfying $t_{i,j} < t_{i, j+1}$. The trajectory of an individual $i$ is defined as $\mathbf{X}_i = [\mathbf{x}_{i,1}, ..., \mathbf{x}_{i,l}]$, where $l$ is the trajectory length.
\end{definition}

\begin{definition}
\label{def:tsw}
\textbf{Trajectory Sequence Window.}  
To process trajectories, we define a sliding window of size $k$ ($k < l$) that partitions a trajectory $\mathbf{X}_i$ into overlapping segments. Each segment is represented as $\mathbf{S}_p = [\mathbf{s}_{p,1}, ..., \mathbf{s}_{p,k}]$, yielding $l-k+1$ segments per trajectory. Given $M$ trajectories, the total number of segments for a fixed $k$ is $\sum_{i}^{M}(l_i - k + 1)$, where trajectories shorter than $k$ are discarded. 
\end{definition}

\begin{definition}
\label{def:mininfo}
\textbf{Minimal-Information Imputation.}  
Given a trajectory $\mathbf{X}_i = [\mathbf{x}_{i,1}, ..., \mathbf{x}_{i,l}]$, where each point $\mathbf{x}_{i,j} = \{t_{i,j}, lon_{i,j}, lat_{i,j}\}$ represents a spatio-temporal coordinate, we define the \textit{minimal-information imputation problem} as reconstructing $\mathbf{x}_{i,2}, ..., \mathbf{x}_{i,l-1}$ given only the endpoints $\mathbf{x}_{i,1}$ and $\mathbf{x}_{i,l}$. 
\end{definition}

\subsection{Base Network Components}
\textbf{Diffusion Base Model. }  
To capture spatiotemporal dependencies in trajectory imputation, we employ a 1D-UNet with residual network (ResNet) blocks. The 1D-UNet consists of down-sampling and up-sampling modules, linked by a self-attention layer. Each module encodes hidden features using group normalization, nonlinear activation, and 1D-CNN layers. The self-attention mechanism refines trajectory representations via:  
\begin{align}
    \text{Self-Attn}(Q_h,K_h,V_h) = \text{Softmax}\left(\frac{Q_hK_h^T}{\sqrt{d_h}}\right)V_h,
\end{align}
where $Q_h$, $K_h$, and $V_h$ are derived from hidden features $\mathbf{h}$. The features obtained through self-attention are then passed through the up-sampling module to output the predicted noise, as shown in the right of Fig. \ref{frame}. 

\begin{figure}[ht]
\centering
\includegraphics[width=1.0\linewidth]{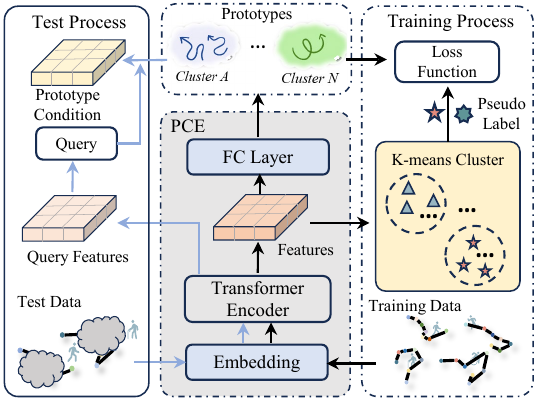}
\caption{Composition of prototype condition extractor and its workflow during the training and test (black and blue lines). }
\label{proto}
\end{figure}

\textbf{Base Condition.}  
Trajectory imputation relies on reconstructing intermediate points from the trajectory endpoints. Specifically, given a set of trajectory points $\mathbf{S}_i = [\mathbf{s}_{i, 1}, ... , \mathbf{s}_{i, k}] \in \mathbb{R}^{k \times d}$, where $k$ denotes the trajectory length, we generate a mask $\mathbf{M} = [\mathbf{m}_{1}, ... , \mathbf{m}_{k}] \in \mathbb{R}^{k}$ which is corresponding to $\mathbf{S}_i$. For any element $\mathbf{m}_{j}$:
\begin{align}
    \mathbf{m}_{j} = 
    \left\{\begin{aligned}
        1&,  \quad \text{if}\,j = 0\,\text{or}\,j = k, \\ 
        0&,  \quad \text{otherwise.}
    \end{aligned}\right.
    \label{eq:base condition}
\end{align}
This mask, when applied to trajectory points, encodes the locations to acquire the base condition $\mathbf{B}^c$ while guiding the diffusion model in reconstructing the missing points.  

\subsection{Prototype Condition Extractor}
\textbf{Embedding Trajectory Data.}  
To exploit large-scale unlabeled data, we introduce a Prototype Condition Extractor (PCE) that embeds trajectories into vector space and extracts latent movement patterns. For each trajectory $\mathbf{S}_i = [\mathbf{s}_{i,1}, ..., \mathbf{s}_{i,k}] \in \mathbb{R}^{k \times d}$ of window size $k$, the trajectory representation $\mathbf{H}_i$ is computed as:
\begin{align}
    \mathbf{H}_i = \sum_{j}^k\left(\text{Encoder}(\mathbf{s}_{i, j})\right).
    \label{eq:s}
\end{align}
Prototypes $\mathbf{P} \in \mathbb{R}^{N_p \times d_p}$ where $N_p$ denotes the number of prototypes and $d_p$ represents the embedded dimension are then generated via a fully connected layer:
\begin{align}
\mathbf{P} = \mathbf{W_p} \mathbf{H_p} + \mathbf{b_p},
\end{align}
where $\mathbf{W_p}$ and $\mathbf{b_p}$ are learnable parameters. Since the hidden feature $\mathbf{H_p}$ is considered as trajectory representation which is summed up from the features of all points in trajectory $\mathbf{S}$. Moreover, prototypes $\mathbf{P}$ are generated by $\mathbf{H_p}$ expressing the generic movement pattern of trajectory $\mathbf{S}$, which are iteratively refined and serve as conditioning features for inference.

\textbf{Conditioning the Diffusion Model.}  
While the base condition serves as a guide, it often provides implicit information, making it difficult for the model to derive sufficient insights directly. To enhance the diffusion model's guidance, we encode trajectory data into queries $\mathbf{Q}^b = \{Q_1, ..., Q_B\} \in \mathbb{R}^{B \times d}$ of $B$ trajectories and project them into the prototype space using:
\begin{align}
    \mathbf{D} &= \left[\text{Dis}(\mathbf{Q}_b, P_1), ..., \text{Dis}(\mathbf{Q}_b, P_{N_p})\right],  \\
    \mathbf{P}^c &= \mathbf{D}^{T} \mathbf{P}.
\end{align}
Here, $\mathbf{P}^c$ represents the prototype-conditioned feature, aligning trajectory embeddings with learned movement patterns. $\text{Dis}(\mathbf{Q}^b, P_i)$ can be an arbitrary distance function between query $\mathbf{Q}^b$ and $i^{th}$ prototype $P_i$. With encoder optimization, the prototypes representing movement patterns are refined, and the PCE can effectively enhance the diffusion model's guidance by matching the base condition with the prototype and generating a comprehensive prototype condition.

To integrate the base condition and prototype condition, we encode and combine them using a Wide \& Deep (WD) network, which contains two fully connected layers for each condition. Then the final joint condition $\mathcal{J}^c$ is formulated as:
\begin{align}
    \mathcal{J}^c = WD(\mathbf{B}^c) + WD(\mathbf{P}^c).
    \label{final condition}
\end{align}

Fig. \ref{proto} details the specific workflow of the prototype network. On the right and middle sections, complete trajectories are used to train the prototype network, enhancing the generation of prototypes that accurately represent movement patterns. This training is optimized through unsupervised contrastive loss and the joint loss function. On the left and middle sections, during testing, trajectories are encoded and used to query the trained prototypes to generate the prototype condition.

\subsection{Jointly Training Objective}
Given i.i.d. samples $\mathbf{Z} \sim p$, a diffusion probabilistic model approximates the data distribution by learning $p_{\theta}(\mathbf{Z})$. In the \textit{forward process}, Gaussian noise diffuses the data via the stochastic differential equation (SDE):
\begin{align}
    \mathrm{d}\mathbf{Z} = \mathbf{f}(\mathbf{Z}, t)dt + g(t)\mathrm{d}\mathbf{w},
\end{align}
where $\mathbf{f}(\cdot)$ is the drift coefficient, $g(\cdot)$ is the diffusion coefficient, and $\mathbf{w}$ is a standard Wiener process. The reverse process, conditioned on $\mathcal{J}^c$, is given by:
\begin{align}
    \mathrm{d}\mathbf{Z}\! = \!\bigl[\mathbf{f}(\mathbf{Z}, t) \!- \!g(t)^2\nabla_\mathbf{Z} \log p_t\left(\mathbf{Z} | \mathcal{J}^c\right)\bigr]dt \!+ \!g(t)\mathrm{d}\bar{\mathbf{w}}.
\end{align}
where $\nabla_\mathbf{Z} \log p_t(\mathbf{Z} | \mathcal{J}^c)$ is the conditional score function. The denoising network $\mathbf{\epsilon}_\theta$ estimates this score function:
\begin{align}
    \mathbf{\epsilon}_\theta(\mathbf{Z}_t, t, \mathcal{J}^c) \simeq - g(t)^2 \nabla_\mathbf{Z} \log p_t(\mathbf{Z} | \mathcal{J}^c),
\end{align}
where the joint condition $\mathcal{J}^c = f_\gamma(\mathbf{Z}_0)$. The joint loss function is:
\begin{align}
    \scalebox{0.93}{$\mathcal{L}_{J}(\theta, \gamma) \!= 
    \!\mathbb{E}_{t \sim \mathcal{U}} \mathbb{E}_{\mathbf{Z}_0 \sim p, \mathbf{\epsilon} \sim \mathcal{N}}
    \left[\lVert \mathbf{\epsilon} \!-\! \mathbf{\epsilon}_{\theta}(\mathbf{Z}_t, t, f_\gamma(\mathbf{Z}_0))\rVert^2\right]$}.
\end{align}
 $\theta$ and $\gamma$ are the optimized parameters of denoising network and joint condition extraction network.

 To enhance prototype learning for unsupervised trajectory data, we introduce additional loss functions to refine $f_\gamma$ and capture semantic movement patterns.
The first classification consistency loss, $\mathcal{L}_{C1}$, enforces alignment between K-means clustering and prototype-based learning. Given trajectory features $\mathbf{H_p} = \{\mathbf{H}_1, \mathbf{H}_2, ...\}$ and $N_c$ clusters, K-means assigns pseudo-labels $p_{kmeans}$, which guide prototype learning via:
\begin{align}
    \mathcal{L}_{C1}(\gamma) = -\sum_{i = 1}^{N_c} p_{kmeans}^i \log(q_{proto}^i),
\end{align}
where $q_{proto}^i$ represents the prototype-assigned label.
To ensure a compact, discriminative feature space, we employ a contrastive loss that optimizes prototype separation. Given trajectory features $\mathbf{H_p}$ and prototypes $\mathbf{P} = \{P_1, ..., P_{N_p}\}$, let $\mathbf{P}^+$ and $\mathbf{P}^-$ be the closest and farthest prototypes, respectively. The loss is defined as:
\begin{align}
    \scalebox{0.97}{$\mathcal{L}_{C2}(\gamma) \!=\! \mathbb{E} \bigl[\max\left(0, d(\mathbf{H}_i, \mathbf{P}^+) \!- \!d(\mathbf{H}_i, \mathbf{P}^-) \!+ \!m\right)\bigl]$},
\end{align}
where $m$ is a margin ensuring separation, and it $d(\cdot, \cdot)$ is a distance metric (e.g., Euclidean).
Afterward, the final objective integrates all loss functions:
\begin{align}
    \mathcal{L}(\theta, \gamma) = \lambda_1 \mathcal{L}_{J}(\theta, \gamma) + \lambda_2 \mathcal{L}_{C1}(\gamma) + \lambda_3 \mathcal{L}_{C2}(\gamma),
\end{align}
where $\lambda_1, \lambda_2, \lambda_3$ control the weight of each term. The full training process is outlined in \Cref{trainalg}.


\subsection{Inference Processes}
In the Inference process, given information about two points in a trajectory with sequential order, the corresponding base condition can be generated according to Eq. \ref{eq:base condition}. The base condition is utilized as a query and projected into the space of trained robust prototypes to obtain the prototype condition, and finally the joint condition is obtained through Eq. \ref{final condition}. Then, the inference process conduct the trained denoising function $\mathcal{\mathbf{\epsilon}_\theta}$ to denoise from a standard Gaussian noise $\mathbf{Z_t}$ step by step. A more detailed algorithm can be found in \Cref{inferalg}.

\subsection{On the Prototype Loss for Trajectory Learning}
Prototype learning can be seen as the combination of clustering and contrastive learning. Formally, for data points $X=\{x_1,...,x_n\}$, the embedding function $f: \mathbb{R}^d \rightarrow \mathbb{R}^m$, and prototypes $\{p_1,...,p_K\}$ are optimized by,
$$
\min_{f,\{p_k\}}\sum_{i=1}^n\Vert f(x_i)-p_{y_i}\Vert ^2+\lambda\ell_{\mathrm{contrast}}(f(x_i),p_{y_i},\{p_k\}).
$$
In trajectory imputation scenario, we propose two basic assumptions which ensure the representiveness of macro-level human movement patterns: (1) Human trajectory data is drawn from a mixture of distributions, each localized on a manifold region $\mathcal{M}$ with mean $\mu_k$. (2) The embedding $f$ enables diverse prototypes that capture local tangents and reconstruct manifold structures via linear combinations \cite{roweis2000nonlinear}. With these assumptions, we have:
\begin{theorem}[The Optimality of Prototype Learning]
\label{thm:pl}
 Any global optimum $(f^*, \{p_k^*\})$ satisfies:
 \begin{enumerate}
    \item Prototypes approximate conditional expectations: $p_k^* \approx \mathbb{E}\left[f^*(x) | x \in C_k\right].$ 
    \item Contrastive loss enforces prototype separation, forming diverse directional vectors: $\langle p_i^*, p_j^* \rangle \leq \epsilon$, for $i\neq j$.
\end{enumerate}
\end{theorem}

\begin{proof}
Using Pollard's consistency theorem \cite{pollard1981strong}, the empirical cluster centers converge to conditional expectations:
$$
p_k^* \approx \mathbb{E}\left[f^*(x) | x \in C_k\right].
$$
From InfoNCE-based contrastive loss \cite{saunshi2019theoretical}, optimality conditions ensure prototype distinctiveness:
$$
f^*(x)^\top p_y^* - f^*(x)^\top p_k^* \ge \delta, \quad \forall k\neq y.
$$

where $ \delta>0$ .Since the clustering term already guarantees that $p_y^* \approx \mathbb{E}\left[f^*(x) \mid x\in C_y\right]$, averaging over cluster $C_y$ gives:
$$
\langle p_y^*, p_y^* \rangle - \langle p_y^*, p_k^* \rangle \ge \delta.
$$
shifting the terms leads to the observation:
$$
\langle p_y^*, p_k^* \rangle \le \|p_y^*\|^2 - \delta \le \epsilon, k\neq y,
$$
which indicates that contrastive loss forces prototypes into globally distinct directions, ensuring effective representation of manifold local structures.
\end{proof}

\begin{table*}[htbp]
  \centering
  \caption{Comparison of model performance for different thresholds and different trajectory lengths on WuXi and FourSquare.}
  \label{tab:main}
  \resizebox{\textwidth}{!}{
  \small 
  \begin{tabular}{clcccccccccc}
    \toprule
    \multirow{2}[4]{*}{\textbf{}} & \multirow{2}[4]{*}{\textbf{Method}} & \multicolumn{5}{c}{\textbf{WuXi}}     & \multicolumn{5}{c}{\textbf{FourSquare}} \\
\cmidrule(r){3-7}     \cmidrule(r){8-12}       &       & \textbf{TC@2k} & \textbf{TC@4k} & \textbf{TC@6k} & \textbf{TC@8k} & \textbf{TC@10k} & \textbf{TC@2k} & \textbf{TC@4k} & \textbf{TC@6k} & \textbf{TC@8k} & \textbf{TC@10k} \\
    \midrule
    \multirow{7}[2]{*}{{\rotatebox{90}{k=4}}} & VAR \cite{lutkepohl2013vector}   & 0.5194 & 0.5632 & 0.6050 & 0.6441 & 0.6811 & 0.5000 & 0.5000 & 0.5000 & 0.5000 & 0.5000 \\
          & SAITS \cite{du2023saits} & 0.5059 & 0.5224 & 0.5498 & 0.5861 & 0.6311 & 0.5000 & 0.5000 & 0.5000 & 0.5000 & 0.5000 \\
          & TimesNet \cite{wu2023timesnet}& 0.5080 & 0.5290 & 0.5593 & 0.5955 & 0.6352 & 0.5015 & 0.5054 & 0.5133 & 0.5258 & 0.5431 \\
          & Diff-TS \cite{yuan2024diffusion}&  0.5123  &  0.5462 &  0.5951  & 0.6496  &  0.7060  & 0.5268   &  0.5714    &  0.6173     &  0.6571     &  0.6932 \\
          & DiffTraj \cite{10.5555/3666122.3668965} &{\color{blue}0.6958}     & {\color{blue}0.8198}      &  {\color{blue}0.8816}     & {\color{blue}0.9169}      &  {\color{blue}0.9402}     &  0.5945      &  0.6845     & 0.7574      &  0.8189     &   0.8666 \\
          & Diff+Mask (Ours) & 0.6584 & 0.7731 & 0.8400 & 0.8834 & 0.9159 & {\color{blue}0.6541} & {\color{blue}0.7379} & {\color{blue}0.8010} & {\color{blue}0.8525} & {\color{blue}0.8928} \\
          & ProDiff (Ours) & {\color{red}\textbf{0.7155}}   &  {\color{red}\textbf{0.8414}}  & {\color{red}\textbf{0.9006}} &  {\color{red}\textbf{0.9326}}  & {\color{red}\textbf{0.9520}}  & {\color{red}\textbf{0.6644}}      &   {\color{red}\textbf{0.7452}}    &  {\color{red}\textbf{0.8087}}     &  {\color{red}\textbf{0.8596}}     & {\color{red}\textbf{0.8971}} \\
    \midrule
    \multirow{7}[2]{*}{{\rotatebox{90}{k=6}}} & VAR \cite{lutkepohl2013vector}   & 0.3360 & 0.3437 & 0.3556 & 0.3692 & 0.3840 & 0.3333 & 0.3333 & 0.3334 & 0.3334 & 0.3335 \\
          & SAITS \cite{du2023saits} & 0.3427 & 0.3762 & 0.4275 & 0.4880 & 0.5533 & 0.3333 & 0.3333 & 0.3333 & 0.3333 & 0.3333 \\
          & TimesNet \cite{wu2023timesnet} & 0.3419 & 0.3654 & 0.4029 & 0.4500 & 0.5044 & 0.3386 & 0.3530 & 0.3756 & 0.4039 & 0.4341 \\
          & Diff-TS \cite{yuan2024diffusion} &  0.3515  &  0.4011  & 0.4726 & 0.5491  & 0.6211   &  0.3761   &  0.4283  &  0.4827    &  0.5383   & 0.5874 \\
          & DiffTraj \cite{10.5555/3666122.3668965}& {\color{blue}0.5976}   &  {\color{blue}0.7476}     &  0.8227     & 0.8688   &  0.9005   & 0.4277    &      0.5404 &    0.6428    &   0.7314    &    0.8025    \\
          & Diff+Mask (Ours) & 0.5767 & 0.7324 & {\color{blue}0.8228} & {\color{blue}0.8802} & {\color{blue}0.9180} & {\color{blue}0.4859} & {\color{blue}0.5970} & {\color{blue}0.6902} & {\color{blue}0.7671} & {\color{blue}0.8265} \\
          & ProDiff (Ours)& {\color{red}\textbf{0.5978}}  &  {\color{red}\textbf{0.7686}}     &  {\color{red}\textbf{0.8518}}     &   {\color{red}\textbf{0.8992}}   &   {\color{red}\textbf{0.9285}}    & {\color{red}\textbf{0.5005}}      &   {\color{red}\textbf{0.6093}}    &  {\color{red}\textbf{0.7013}}     &  {\color{red}\textbf{0.7772}}     &  {\color{red}\textbf{0.8345}} \\
    \midrule
    \multirow{7}[2]{*}{{\rotatebox{90}{k=8}}} & VAR \cite{lutkepohl2013vector}   & 0.2537 & 0.2627 & 0.2739 & 0.2861 & 0.2986 & 0.2500 & 0.2500 & 0.2500 & 0.2500 & 0.2500 \\
          & SAITS \cite{du2023saits}& 0.2572 & 0.2764 & 0.3059 & 0.3485 & 0.3976 & 0.2500 & 0.2502 & 0.2505 & 0.2509 & 0.2513 \\
          & TimesNet \cite{wu2023timesnet}& 0.2520 & 0.2574 & 0.2663 & 0.2785 & 0.2942 & 0.2516 & 0.2563 & 0.2634 & 0.2715 & 0.2808 \\
          & Diff-TS \cite{yuan2024diffusion}&  0.2689 & 0.3199  &0.3907   & 0.4676  &0.5453      &  0.3233     &    0.3932   &  0.4611    &  0.5358   & 0.5964 \\
          & DiffTraj \cite{10.5555/3666122.3668965}& {\color{blue}0.5418}    &  {\color{blue}0.7009}     &   {\color{blue}0.7868}    &     {\color{blue}0.8414}   &  {\color{blue}0.8795}     & 0.3316     &   0.4526    &   0.5671    &  0.6688     & 0.7520     \\
          & Diff+Mask (Ours)& 0.4486 & 0.5946 & 0.6943 & 0.7631 & 0.8107 & {\color{blue}0.3957} & {\color{blue}0.5300} & {\color{blue}0.6431} & {\color{blue}0.7350} & {\color{blue}0.8045} \\
          & ProDiff (Ours)&  {\color{red}\textbf{0.5752}}  &  {\color{red}\textbf{0.7501}}     & {\color{red}\textbf{ 0.8236}}     &  {\color{red}\textbf{0.8663}}     &{\color{red}\textbf{ 0.8945}}      &  {\color{red}\textbf{0.4000}}    &  {\color{red}\textbf{0.5331}}     & {\color{red}\textbf{0.6474}}      &{\color{red}\textbf{0.7404}}       & {\color{red}\textbf{0.8090}}  \\
    \midrule
    \multirow{7}[2]{*}{{\rotatebox{90}{k=10}}} & VAR \cite{lutkepohl2013vector}   & 0.2012 & 0.2047 & 0.2102 & 0.2177 & 0.2270 & 0.2000 & 0.2000 & 0.2000 & 0.2000 & 0.2000 \\
          & SAITS \cite{du2023saits} & 0.2080 & 0.2316 & 0.2686 & 0.3158 & 0.3692 & 0.2000 & 0.2000 & 0.2000 & 0.2000 & 0.2000 \\
          & TimesNet \cite{wu2023timesnet}& 0.2073 & 0.2275 & 0.2591 & 0.3035 & 0.3559 & 0.2003 & 0.2013 & 0.2034 & 0.2064 & 0.2110 \\
          & Diff-TS \cite{yuan2024diffusion}&  0.2173 & 0.2655  & 0.3367  & 0.4190  &  0.5000     &  0.2751     & 0.3484   & 0.4207   & 0.4990   & 0.5646 \\
          & DiffTraj \cite{10.5555/3666122.3668965}&  {\color{blue}0.4994}     & {\color{blue}0.6687}    &  {\color{blue}0.7640}     &   {\color{blue}0.8259}    &  {\color{blue}0.8692}     &  0.2762     &  0.4024     & 0.5300      &  0.6453     & 0.7386 \\
          & Diff+Mask (Ours)& 0.3793 & 0.5104 & 0.6046 & 0.6773 & 0.7344 & {\color{blue}0.3412}  & {\color{blue}0.4800} & {\color{blue}0.5999}  & {\color{blue}0.7023}  &  {\color{blue}0.7868} \\
          & ProDiff (Ours)& {\color{red}\textbf{0.4996}}      & {\color{red}\textbf{0.6994}}         & {\color{red}\textbf{0.8048}}      &  {\color{red}\textbf{0.8667}}     &  {\color{red}\textbf{0.9053}}     &  {\color{red}\textbf{0.3522}}     &   {\color{red}\textbf{0.4910}}    &   {\color{red}\textbf{0.6105}}      &   {\color{red}\textbf{0.7146}}    &  {\color{red}\textbf{0.7920}} \\
    \bottomrule
  \end{tabular}%
  }
\end{table*}

\section{Experiments}
\subsection{Datasets}
Our experiments utilize two well-established trajectory datasets:  
(1) WuXi: Extracted from mobile signal data \cite{song2017recovering}, covering WuXi, China, over six months (Oct 2013–Mar 2014). It records locations whenever users’ phones are active. For efficiency, we use a 10-day subset, concatenating individual trajectories.  
(2) Foursquare: A public dataset \cite{yang2014modeling} containing check-ins over 10 months (Apr 2012–Feb 2013) in New York and Tokyo. Each check-in includes a timestamp, GPS coordinates, and semantic tags.  
All datasets are anonymized, ensuring no privacy concerns. Tab. \ref{tab:hm} in \cref{sec:appendix:experiment} provides details. 

\subsection{Evaluation Metric and Baseline}


We evaluate trajectory imputation by comparing against (i) time-series interpolation methods and (ii) trajectory-specific approaches, with corresponding evaluation metrics.

For time-series interpolation, we benchmark against classical methods like Vector Autoregression (VAR) \cite{lutkepohl2005new} and state-of-the-art spatio-temporal models, including SAITS \cite{du2023saits}, TimesNet \cite{wu2023timesnet}, Diffusion-TS \cite{yuan2024diffusion}, and DiffTraj \cite{10.5555/3666122.3668965}. To measure imputation accuracy, we introduce trajectory coverage ($TC@\tau$), which quantifies the proportion of generated points $\hat{\mathbf{s}_{i, j}}$ within a threshold $\tau$ from the ground truth ${\mathbf{s}_{i, j}}$:
\begin{align}
    TC@\tau = \frac{1}{k}\sum_{j=1}^{k} \mathbb{I}\left(d(\hat{\mathbf{s}_{i, j}}, \mathbf{s}_{i, j}) < \tau\right),
\end{align}
where $\hat{\mathbf{s}_{i, j}}$ is the generated points, and $\mathbb{I}(\cdot)$ is an indicator function that equals when $d(\hat{\mathbf{s}_{i, j}}, \mathbf{s}_{i, j})$ is less than the threshold $\tau$ and $0$ otherwise.

For trajectory-specific baselines, we compare against RNTrajRec \cite{chen2023rntrajrec}, TS-TrajGen \cite{jiang2023continuous}, MM-STGED \cite{wei2024micro}, AttnMove \cite{xia2021attnmove}, and DiffTraj \cite{10.5555/3666122.3668965}. To ensure fairness, we remove modules reliant on unavailable auxiliary information. Performance is assessed using standard trajectory generation metrics, including Density, Distance, Segment Distance, Radius, MAE, and RMSE. Further details on evaluation protocols and baselines are provided in \Cref{sec:appendix:experiment}.

\subsection{Implementation Details}

Our experiments balance the effectiveness of each module in the joint training; we set $\lambda_{1}, \lambda_{2}, \lambda_{3}$ to 1. Gradient updates were facilitated using the Adam optimizer, initialized with a learning rate of $2e^{-4}$.
We summarize the hyperparameter settings for the diffusion model and PCE in Tab. \ref{tab:hp_2}.

\begin{table}[htbp]
  \centering
  \caption{General setting of ProDiff model}
    \resizebox{\linewidth}{!}{
    \begin{tabular}{cccc}
    \toprule
    \multicolumn{2}{c}{\textbf{Diffusion}}     & \multicolumn{2}{c}{\textbf{PCE}} \\
    \cmidrule(r){1-2}     \cmidrule(r){3-4} 
    \multicolumn{1}{c}{Parameter} & \multicolumn{1}{c}{Setting value} & \multicolumn{1}{c}{Parameter} & \multicolumn{1}{c}{Setting value}\\
    \midrule
    Diffusion Steps  & 500  & Prototypes & 20 \\
    Embedding Dim & 128 & Embedding Dim & 512 \\
    $\beta$ (linear schedule) & 0.0001$\sim$ 0.05  & Heads & 8\\
    ResNet Blocks & 2  & Encoder Blocks & 4\\
    Sampling Blocks & 4  & Forward Dim & 256\\
    Input Length & 3 $\sim$ 10  & Dropout ratio & 0.1 \\
    \bottomrule
    \end{tabular}%
    }
  \label{tab:hp_2}%
\end{table}%

\begin{figure*}[ht]
\centering
\includegraphics[width=1.0\linewidth]{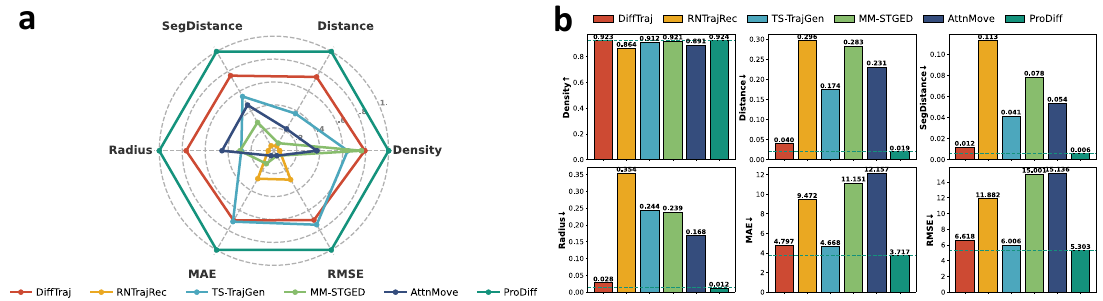}
\caption{\textbf{a}. Radar charts illustrate the normalized performance of different models across six distinct metrics. \textbf{b}. Histogram comparing the performance of each model across different metrics, with dashed lines indicating the best-performing model's values for each metric.}
\label{ext}
\end{figure*}

\subsection{Main Results}
The trajectory coverage across different baselines and window sizes is presented in Tab. \ref{tab:main}, where “TC@2k” represents the percentage of generated values within 2km of the true location, with TC@4k–10k extending up to 10km. The highest and second-highest values are marked in red and blue, respectively. We evaluate performance separately for time-series interpolation methods and trajectory-specific approaches.

\textbf{(1) Comparison with Time-Series Interpolation Methods.}  
ProDiff consistently outperforms sequence imputation models across datasets. On the WuXi dataset with $k=4$, ProDiff achieves 71.55\% at TC@2k, exceeding all baselines ($<$70\%). As the threshold increases (TC@4k–10k), ProDiff maintains high accuracy (84.14\%–95.20\%), with its advantage over the second-best method expanding from 1.18\% at TC@2k to 3.61\% at TC@10k. 
Furthermore, ProDiff demonstrates robustness across datasets and segment sizes, where other models, such as DiffTraj, suffer sharp declines at larger thresholds (TC@6k–10k). On the FourSquare dataset, DiffTraj’s TC@2k score for $k=8$ drops by 21.02\%, while ProDiff only decreases by 11.59\%. Additionally, while DiffTraj loses its second-place ranking to Diff+Mask, ProDiff retains its lead, indicating its ability to learn stable movement patterns via the prototype condition extractor.

\textbf{(2) Comparison with Trajectory-Specific Methods.}  
Fig. \ref{ext} further validates ProDiff’s superiority among trajectory models. Panel (a) presents normalized scores across all metrics, while panel (b) details model-specific performance. ProDiff consistently sets the benchmark across six additional metrics. While density scores are similar among models, ProDiff exhibits a substantial lead in spatial distribution metrics (e.g., Distance, Segment Distance, Radius, MAE, RMSE), highlighting its effectiveness in diverse conditions.

\subsection{Ablation Study}
We conducted three ablation experiments on the WuXi dataset to validate the contributions of our key components. 
We also investigate the accerlaration of the proposed ProDiff, and the results are provided in \Cref{sec:appendix:experiment}.

\textbf{Effect of Prototype Condition Extractor.}  
To assess the impact of individual modules, we removed the prototype condition extractor (PCE), cross-entropy loss ($\mathcal{L}_{C1}$), and contrastive loss ($\mathcal{L}_{C2}$) while keeping the joint loss intact. As shown in Tab. \ref{tab:ab1}, PCE consistently improves performance, with $\mathcal{L}_{C1}$ and $\mathcal{L}_{C2}$ further enhancing its effectiveness in capturing movement patterns. Notably, the performance gain of PCE is more pronounced at longer distances when $k=8$, suggesting that its effectiveness increases with extended trajectory segments.

\begin{table}[htbp]
  \centering
  \caption{Performance comparison of removing different modules.}
  \resizebox{\linewidth}{!}{
    \begin{tabular}{clrrrrr}
    \toprule
    \multicolumn{1}{l}{} & \multicolumn{1}{l}{Method} & \multicolumn{1}{c}{TC@2k} & \multicolumn{1}{c}{TC@4k} & \multicolumn{1}{c}{TC@6k} & \multicolumn{1}{c}{TC@8k} & \multicolumn{1}{c}{TC@10k} \\
    \midrule
    \multirow{4}[2]{*}{{\rotatebox{90}{k=6}}} &
    ProDiff & \textbf{0.5978}  &  \textbf{0.7686}     &  \textbf{0.8518}     &   \textbf{0.8992}   &   \textbf{0.9285}  \\
    & w.o. Pro & 0.5767 & 0.7324 & 0.8228 & 0.8802 & 0.9180 \\
    & w.o. $\mathcal{L}_{C1}$ &   0.5939    & 0.7556      & 0.8371      &  0.8867     &  0.9195 \\
    & w.o. $\mathcal{L}_{C2}$ &   0.5952    &  0.7560     &  0.8374     &  0.8869     &  0.9199 \\
        \midrule
    \multirow{4}[2]{*}{{\rotatebox{90}{k=8}}} &
    ProDiff & \textbf{0.5752}  &  \textbf{0.7501}     &  \textbf{0.8236}     &   \textbf{0.8663}   &   \textbf{0.8945}  \\
    & w.o. Pro& 0.4486 & 0.5946 & 0.6943 & 0.7631 & 0.8107 \\
    & w.o. $\mathcal{L}_{C1}$ &  0.5395    &  0.7205   &  0.7966   & 0.8399    & 0.8691 \\
    & w.o. $\mathcal{L}_{C2}$ & 0.4888    &  0.6638    &  0.7473    & 0.7984     & 0.8340 \\
    \bottomrule
    \end{tabular}%
  }
  \label{tab:ab1}%
\end{table}%

\begin{table}[htbp]
  \centering
  \caption{Performance comparison for cVAE and cGAN.}
    \resizebox{\linewidth}{!}{
    \begin{tabular}{lrrrrr}
    \toprule
    \multicolumn{1}{c}{Method} & \multicolumn{1}{c}{TC@2k} & \multicolumn{1}{c}{TC@4k} & \multicolumn{1}{c}{TC@6k} & \multicolumn{1}{c}{TC@8k} & \multicolumn{1}{c}{TC@10k} \\
    \midrule
    cVAE+MASK  & 0.2616      &  0.2936     & 0.3385      &  0.3926     &  0.4513 \\
    cVAE+pro & 0.3416    &  0.3685     &  0.4082     & 0.4540      & 0.5009  \\
    cGAN+MASK  & 0.2760   & 0.3240   &  0.3742  &  0.4285  & 0.4896  \\
    cGAN+pro & 0.3074   &  0.3997     & 0.4746      & 0.5361      & 0.5883 \\
    \bottomrule
    \end{tabular}%
    }
  \label{tab:ab2}%
\end{table}%

\textbf{Generalization Across Generative Models.}  
To evaluate PCE's generalizability, we integrated it with cVAE and cGAN, applying both MASK and PCE to these models (Tab. \ref{tab:ab2}). At the 10k threshold, adding only MASK yields 45.13\% for cVAE and 48.96\% for cGAN, whereas incorporating PCE improves accuracy to 50.09\% and 58.83\%, respectively. This highlights PCE’s ability to enhance movement pattern learning across different generative frameworks.

\begin{figure*}[ht]
	\centering
	\includegraphics[width=0.95\linewidth]{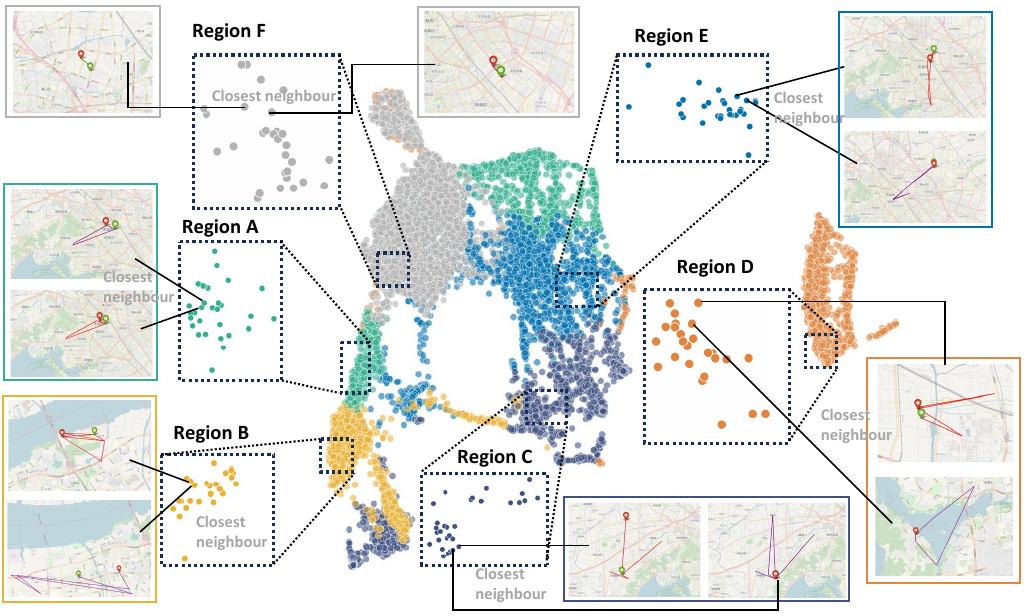}
	\caption{Trajectory data representation after dimensionality reduction by PaCMAP, randomly selected samples and neighboring samples plot trajectories to interpret human trajectory patterns captured by prototype learning. }
	\label{vis}
\end{figure*}

\subsection{Utility of Generated Data}
To evaluate the real-world applicability of ProDiff generated data, we tested its performance on traffic flow analysis in WuXi, using $k=6$ trajectory imputations over 7000 individuals across 10 days. The city was divided into 1km $\times$ 1km grids (longitude gap $\approx 0.009^\circ$), where each grid's value increments as individuals’ trajectories pass through.  
Fig. \ref{flow}(a) (top) compares real and ProDiff-generated traffic maps, revealing highly similar spatial patterns. To further analyze peak and trend consistency, we extracted and projected traffic edges (Fig. \ref{flow}(a), bottom), showing near-identical fluctuations between real and generated data.  
Additionally, correlation coefficients and spatial distributions between real and generated data (Fig. \ref{flow}(b), \ref{flow}(c)) further confirm the reliability of ProDiff's imputation. These results demonstrate that ProDiff can generate realistic and usable trajectory data, making it applicable to downstream mobility analysis tasks.

\begin{table}[htbp]
	\centering
	\caption{Impact of different numbers of prototypes (N) and trajectory length (k).}
	\resizebox{\linewidth}{!}{
		\begin{tabular}{clrrrrr}
			\toprule
			\multicolumn{1}{l}{} & \multicolumn{1}{l}{N} & \multicolumn{1}{c}{TC@2k} & \multicolumn{1}{c}{TC@4k} & \multicolumn{1}{c}{TC@6k} & \multicolumn{1}{c}{TC@8k} & \multicolumn{1}{c}{TC@10k} \\
			\midrule
			\multirow{3}[2]{*}{{\rotatebox{90}{k=6}}} &
			15 & 0.5881  &  0.7583     &  0.8433     &   0.8928   &  0.9237  \\
			& 20 & 0.5978  &  0.7686    &  0.8518     &   0.8992   &  0.9285 \\
			& 25 &   0.5868    & 0.7570      & 0.8441      &  0.8951     &  0.9265 \\
			\midrule
			\multirow{3}[2]{*}{{\rotatebox{90}{k=8}}} &
			15 & 0.5755  &  0.7552     &  0.8217     &   0.8763   &  0.8951  \\
			&20& 0.5752 & 0.7501 & 0.8236 & 0.8663 & 0.8945 \\
			& 25 &  0.5785    &  0.7553   &  0.8237   & 0.8634    & 0.9017 \\
			\bottomrule
		\end{tabular}%
	}
	\label{tab:nk}%
\end{table}%

\begin{figure}[htbp]
	\centering
	\includegraphics[width=0.9\linewidth]{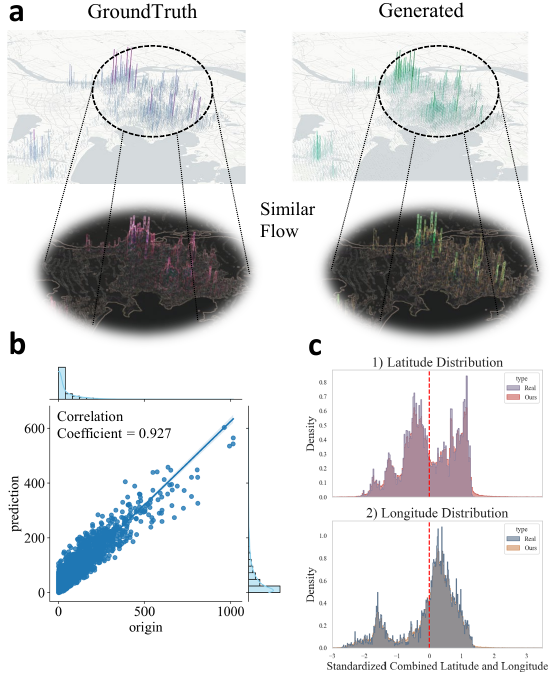}
	\caption{a. Comparison of traffic patterns between groundtruth and generated data. b. The correlation coefficient between groundtruth and generated data. c. Comparison of spatial distributions after normalization of both real and generated data. }
	\label{flow}
\end{figure}

\subsection{Hyperparameter Sensitivity.}  
We analyze the effect of prototype count ($N$), trajectory length ($k$), and diffusion steps ($d$) (Tab. \ref{tab:nk}, Appendix Tab. \ref{tab:d}). Increasing $N$ from 15 to 20 improves TC@10k to 0.9285 for $k=6$, while $k=8$ benefits from $N=25$, suggesting behavioral variations across window sizes. Diffusion steps significantly affect performance, with 300 steps yielding optimal TC@10k (0.9300). Beyond this, performance plateaus, while computational cost increases, making 300 steps a practical balance between accuracy and efficiency.

\subsection{Interpretability Analysis}
Understanding whether prototype learning captures interpretable movement patterns in low-dimensional space is essential for evaluating the effectiveness of the joint prototype learning-diffusion framework. Fig. \ref{vis} visualizes this process, where trajectory data is fed into the trained prototype condition extractor, clustered using K-means (top 6 classes), and reduced in dimensionality via PaCMAP. To interpret the latent space, we zoom into each class, plotting selected samples and their nearest neighbors.  
The learned movement patterns exhibit clear semantic coherence. Region A captures trajectories with start and end points in close proximity, reflecting movement within similar locations. Region B extends this pattern, with slightly farther start and end points, aligning with the proximity of the yellow and green clusters. Region C represents trajectories constrained within similar locations. Region D deviates from previous patterns, showing long-distance migration with a return to the starting point. Region E follows a linear migration and return pattern. Region F is similar to Region A, but its neighboring trajectories occur in different locations. These findings demonstrate the model’s ability to embed human trajectory, capture structured movement behaviors, distinguish variations, and optimize representations during training, improving trajectory imputation performance when integrated with the diffusion framework.

\section{Conclusion}

This paper addresses the trajectory imputation problem, focusing on generating realistic trajectories with minimal information. Unlike conventional methods that rely on sparse trajectory pattern, we propose ProDiff, a prototype-guided diffusion model that captures macro-level mobility patterns while maintaining high fidelity in trajectory generation.
Our experiments demonstrate that ProDiff outperforms state-of-the-art approaches on two datasets, improving trajectory imputation accuracy. Ablation studies confirm that prototype learning significantly enhances trajectory representation, while diffusion modeling effectively reconstructs realistic movements.
Beyond imputation, ProDiff may be generalized to broader trajectory-related tasks, offering a scalable solution for urban mobility analysis and behavioral modeling. Moving forward, we aim to extend ProDiff to adaptive and personalized trajectory generation, integrating reinforcement learning and uncertainty-aware models to enhance reliability under dynamic and noisy conditions.

\section*{Acknowledgments}
Prof. Xin Lu's work was supported by the National Natural Science Foundation of China (72025405, 72421002, 92467302, 72474223, 72301285), the Science and Technology Innovation Program of Hunan Province (2023JJ40685, 2024RC3133), and the Major Program of Xiangjiang Laboratory (24XJJCYJ01001). Prof. Jingyuan Wang's work was partially supported by the National Natural Science Foundation of China (No. 72222022, 72171013). Suoyi Tan was supported by the National Nature Science Foundation of China (No. 72474223, 72001211 ), the science and technology innovation Program of Hunan Province (No.  2024RC3133),  and the National University of Defense Technology Cornerstone Project (No. JS24-04).

\section*{Impact Statement}
Trajectory Imputation is essential for dealing with incomplete trajectory data, a common issue stemming from device limitations and varied collection scenarios.  Our work presents ProDiff, a prototype-guided diffusion model, to effectively impute trajectories using only minimal information. This approach allows for robust spatiotemporal reconstruction of human movement, even from highly sparse data.  While there will be important impacts resulting from improved trajectory imputation in general, here we focus on the impact of using our ProDiff framework for minimal information imputation.  There are many benefits to using our method, such as significantly improving imputation accuracy and effectively embedding and capturing macro-level human movement patterns.  This paper presents work whose goal is to advance the field of Trajectory Data Mining and Spatio-temporal Data Analysis.  There are many potential societal consequences of our work, none which we feel must be specifically highlighted here.

\nocite{langley00}

\bibliography{example_paper}

\begin{thebibliography}{61}
\providecommand{\natexlab}[1]{#1}
\providecommand{\url}[1]{\texttt{#1}}
\expandafter\ifx\csname urlstyle\endcsname\relax
  \providecommand{\doi}[1]{doi: #1}\else
  \providecommand{\doi}{doi: \begingroup \urlstyle{rm}\Url}\fi

\bibitem[Bao et~al.(2021)Bao, Huang, Li, Wang, and Liu]{bao2021bilstm}
Bao, Y., Huang, Z., Li, L., Wang, Y., and Liu, Y.
\newblock A bilstm-cnn model for predicting users’ next locations based on geotagged social media.
\newblock \emph{International Journal of Geographical Information Science}, 35\penalty0 (4):\penalty0 639--660, 2021.

\bibitem[Blu et~al.(2004)Blu, Th{\'e}venaz, and Unser]{blu2004linear}
Blu, T., Th{\'e}venaz, P., and Unser, M.
\newblock Linear interpolation revitalized.
\newblock \emph{IEEE Transactions on Image Processing}, 13\penalty0 (5):\penalty0 710--719, 2004.

\bibitem[Che et~al.(2018)Che, Purushotham, Cho, Sontag, and Liu]{che2018recurrent}
Che, Z., Purushotham, S., Cho, K., Sontag, D., and Liu, Y.
\newblock Recurrent neural networks for multivariate time series with missing values.
\newblock \emph{Scientific reports}, 8\penalty0 (1):\penalty0 6085, 2018.

\bibitem[Chen et~al.(2024)Chen, Liang, Zhu, Chang, Luo, Wen, Li, Yu, Wen, Chen, et~al.]{chen2024deep}
Chen, W., Liang, Y., Zhu, Y., Chang, Y., Luo, K., Wen, H., Li, L., Yu, Y., Wen, Q., Chen, C., et~al.
\newblock Deep learning for trajectory data management and mining: A survey and beyond.
\newblock \emph{arXiv preprint arXiv:2403.14151}, 2024.

\bibitem[Chen et~al.(2021)Chen, Xu, Zhou, Chen, Fang, and Liu]{chen2021trajvae}
Chen, X., Xu, J., Zhou, R., Chen, W., Fang, J., and Liu, C.
\newblock Trajvae: A variational autoencoder model for trajectory generation.
\newblock \emph{Neurocomputing}, 428:\penalty0 332--339, 2021.

\bibitem[Chen et~al.(2023)Chen, Zhang, Sun, and Zheng]{chen2023rntrajrec}
Chen, Y., Zhang, H., Sun, W., and Zheng, B.
\newblock Rntrajrec: Road network enhanced trajectory recovery with spatial-temporal transformer.
\newblock In \emph{2023 IEEE 39th International Conference on Data Engineering (ICDE)}, pp.\  829--842. IEEE, 2023.

\bibitem[De~Montjoye et~al.(2013)De~Montjoye, Hidalgo, Verleysen, and Blondel]{de2013unique}
De~Montjoye, Y.-A., Hidalgo, C.~A., Verleysen, M., and Blondel, V.~D.
\newblock Unique in the crowd: The privacy bounds of human mobility.
\newblock \emph{Scientific reports}, 3\penalty0 (1):\penalty0 1--5, 2013.

\bibitem[Doersch(2016)]{doersch2016tutorial}
Doersch, C.
\newblock Tutorial on variational autoencoders.
\newblock \emph{arXiv preprint arXiv:1606.05908}, 2016.

\bibitem[Du et~al.(2023)Du, C{\^o}t{\'e}, and Liu]{du2023saits}
Du, W., C{\^o}t{\'e}, D., and Liu, Y.
\newblock Saits: Self-attention-based imputation for time series.
\newblock \emph{Expert Systems with Applications}, 219:\penalty0 119619, 2023.

\bibitem[Goodfellow et~al.(2020)Goodfellow, Pouget-Abadie, Mirza, Xu, Warde-Farley, Ozair, Courville, and Bengio]{goodfellow2020generative}
Goodfellow, I., Pouget-Abadie, J., Mirza, M., Xu, B., Warde-Farley, D., Ozair, S., Courville, A., and Bengio, Y.
\newblock Generative adversarial networks.
\newblock \emph{Communications of the ACM}, 63\penalty0 (11):\penalty0 139--144, 2020.

\bibitem[Hettige et~al.(2024)Hettige, Ji, Xiang, Long, Cong, and Wang]{hettigeairphynet}
Hettige, K.~H., Ji, J., Xiang, S., Long, C., Cong, G., and Wang, J.
\newblock Airphynet: Harnessing physics-guided neural networks for air quality prediction.
\newblock In \emph{The 12th International Conference on Learning Representations}, 2024.

\bibitem[Ho et~al.(2020)Ho, Jain, and Abbeel]{ho2020denoising}
Ho, J., Jain, A., and Abbeel, P.
\newblock Denoising diffusion probabilistic models.
\newblock \emph{Advances in neural information processing systems}, 33:\penalty0 6840--6851, 2020.

\bibitem[Huang et~al.(2022)Huang, Yang, Chen, Zhang, Wang, and He]{huang2022context}
Huang, L., Yang, Y., Chen, H., Zhang, Y., Wang, Z., and He, L.
\newblock Context-aware road travel time estimation by coupled tensor decomposition based on trajectory data.
\newblock \emph{Knowledge-Based Systems}, 245:\penalty0 108596, 2022.

\bibitem[Huang et~al.(2023)Huang, Huang, Li, and Cui]{huang2023robust}
Huang, L., Huang, J., Li, H., and Cui, J.
\newblock Robust spatial temporal imputation based on spatio-temporal generative adversarial nets.
\newblock \emph{Knowledge-Based Systems}, 279:\penalty0 110919, 2023.

\bibitem[Ji et~al.(2022{\natexlab{a}})Ji, Wang, Jiang, Jiang, and Zhang]{ji2022stden}
Ji, J., Wang, J., Jiang, Z., Jiang, J., and Zhang, H.
\newblock Stden: Towards physics-guided neural networks for traffic flow prediction.
\newblock In \emph{Proceedings of the AAAI conference on artificial intelligence}, volume~36, pp.\  4048--4056, 2022{\natexlab{a}}.

\bibitem[Ji et~al.(2022{\natexlab{b}})Ji, Wang, Wu, Han, Zhang, and Zheng]{ji2022precision}
Ji, J., Wang, J., Wu, J., Han, B., Zhang, J., and Zheng, Y.
\newblock Precision cityshield against hazardous chemicals threats via location mining and self-supervised learning.
\newblock In \emph{Proceedings of the 28th ACM SIGKDD Conference on Knowledge Discovery and Data Mining}, pp.\  3072--3080, 2022{\natexlab{b}}.

\bibitem[Ji et~al.(2023)Ji, Wang, Huang, Wu, Xu, Wu, Zhang, and Zheng]{ji2023spatio}
Ji, J., Wang, J., Huang, C., Wu, J., Xu, B., Wu, Z., Zhang, J., and Zheng, Y.
\newblock Spatio-temporal self-supervised learning for traffic flow prediction.
\newblock In \emph{Proceedings of the AAAI conference on artificial intelligence}, volume~37, pp.\  4356--4364, 2023.

\bibitem[Jia et~al.(2020)Jia, Lu, Yuan, Xu, Jia, and Christakis]{jia2020population}
Jia, J.~S., Lu, X., Yuan, Y., Xu, G., Jia, J., and Christakis, N.~A.
\newblock Population flow drives spatio-temporal distribution of covid-19 in china.
\newblock \emph{Nature}, 582\penalty0 (7812):\penalty0 389--394, 2020.

\bibitem[Jiang et~al.(2023{\natexlab{a}})Jiang, Han, Zhao, and Wang]{jiang2023pdformer}
Jiang, J., Han, C., Zhao, W.~X., and Wang, J.
\newblock Pdformer: Propagation delay-aware dynamic long-range transformer for traffic flow prediction.
\newblock In \emph{Proceedings of the AAAI conference on artificial intelligence}, volume~37, pp.\  4365--4373, 2023{\natexlab{a}}.

\bibitem[Jiang et~al.(2023{\natexlab{b}})Jiang, Zhao, Wang, and Jiang]{jiang2023continuous}
Jiang, W., Zhao, W.~X., Wang, J., and Jiang, J.
\newblock Continuous trajectory generation based on two-stage gan.
\newblock In \emph{Proceedings of the AAAI Conference on Artificial Intelligence}, volume~37, pp.\  4374--4382, 2023{\natexlab{b}}.

\bibitem[Katharopoulos et~al.(2020)Katharopoulos, Vyas, Pappas, and Fleuret]{katharopoulos2020transformers}
Katharopoulos, A., Vyas, A., Pappas, N., and Fleuret, F.
\newblock Transformers are rnns: Fast autoregressive transformers with linear attention.
\newblock In \emph{International conference on machine learning}, pp.\  5156--5165. PMLR, 2020.

\bibitem[Langley(2000)]{langley00}
Langley, P.
\newblock Crafting papers on machine learning.
\newblock In Langley, P. (ed.), \emph{Proceedings of the 17th International Conference on Machine Learning (ICML 2000)}, pp.\  1207--1216, Stanford, CA, 2000. Morgan Kaufmann.

\bibitem[Li et~al.(2023)Li, Li, Ren, Chen, Yuan, and Wang]{bcdiff}
Li, R., Li, C., Ren, D., Chen, G., Yuan, Y., and Wang, G.
\newblock Bcdiff: Bidirectional consistent diffusion for instantaneous trajectory prediction.
\newblock \emph{Advances in Neural Information Processing Systems}, 36:\penalty0 14400--14413, 2023.

\bibitem[Li et~al.(2018)Li, Fu, Wang, Shahabi, Ye, and Liu]{li2018multi}
Li, Y., Fu, K., Wang, Z., Shahabi, C., Ye, J., and Liu, Y.
\newblock Multi-task representation learning for travel time estimation.
\newblock In \emph{Proceedings of the 24th ACM SIGKDD international conference on knowledge discovery \& data mining}, pp.\  1695--1704, 2018.

\bibitem[Liu et~al.(2020)Liu, Zhao, Cong, and Bao]{liu2020online}
Liu, Y., Zhao, K., Cong, G., and Bao, Z.
\newblock Online anomalous trajectory detection with deep generative sequence modeling.
\newblock In \emph{2020 IEEE 36th International Conference on Data Engineering (ICDE)}, pp.\  949--960. IEEE, 2020.

\bibitem[Liu et~al.(2024)Liu, Wang, Li, and He]{liu2024full}
Liu, Z., Wang, J., Li, Z., and He, Y.
\newblock Full bayesian significance testing for neural networks in traffic forecasting.
\newblock In \emph{Proceedings of the 33rd International Joint Conference on Artificial Intelligence (IJCAI)}, 2024.

\bibitem[Lu et~al.(2012)Lu, Bengtsson, and Holme]{lu2012predictability}
Lu, X., Bengtsson, L., and Holme, P.
\newblock Predictability of population displacement after the 2010 haiti earthquake.
\newblock \emph{Proceedings of the National Academy of Sciences}, 109\penalty0 (29):\penalty0 11576--11581, 2012.

\bibitem[L{\"u}tkepohl(2005)]{lutkepohl2005new}
L{\"u}tkepohl, H.
\newblock \emph{New introduction to multiple time series analysis}.
\newblock Springer Science \& Business Media, 2005.

\bibitem[L{\"u}tkepohl(2013)]{lutkepohl2013vector}
L{\"u}tkepohl, H.
\newblock Vector autoregressive models.
\newblock In \emph{Handbook of research methods and applications in empirical macroeconomics}, pp.\  139--164. Edward Elgar Publishing, 2013.

\bibitem[Ma et~al.(2019)Ma, Shou, Zareian, Mansour, Vetro, and Chang]{ma2019cdsa}
Ma, J., Shou, Z., Zareian, A., Mansour, H., Vetro, A., and Chang, S.-F.
\newblock Cdsa: cross-dimensional self-attention for multivariate, geo-tagged time series imputation.
\newblock \emph{arXiv preprint arXiv:1905.09904}, 2019.

\bibitem[Pappalardo et~al.(2023)Pappalardo, Manley, Sekara, and Alessandretti]{Pappalardo_2023}
Pappalardo, L., Manley, E., Sekara, V., and Alessandretti, L.
\newblock Future directions in human mobility science.
\newblock \emph{Nature Computational Science}, 3\penalty0 (7):\penalty0 588–600, July 2023.
\newblock ISSN 2662-8457.
\newblock \doi{10.1038/s43588-023-00469-4}.
\newblock URL \url{http://dx.doi.org/10.1038/s43588-023-00469-4}.

\bibitem[Pollard(1981)]{pollard1981strong}
Pollard, D.
\newblock Strong consistency of k-means clustering.
\newblock \emph{The annals of statistics}, pp.\  135--140, 1981.

\bibitem[Qu et~al.(2009)Qu, Li, Zhang, and Hu]{qu2009ppca}
Qu, L., Li, L., Zhang, Y., and Hu, J.
\newblock Ppca-based missing data imputation for traffic flow volume: A systematical approach.
\newblock \emph{IEEE Transactions on intelligent transportation systems}, 10\penalty0 (3):\penalty0 512--522, 2009.

\bibitem[Roweis \& Saul(2000)Roweis and Saul]{roweis2000nonlinear}
Roweis, S.~T. and Saul, L.~K.
\newblock Nonlinear dimensionality reduction by locally linear embedding.
\newblock \emph{science}, 290\penalty0 (5500):\penalty0 2323--2326, 2000.

\bibitem[Saunshi et~al.(2019)Saunshi, Plevrakis, Arora, Khodak, and Khandeparkar]{saunshi2019theoretical}
Saunshi, N., Plevrakis, O., Arora, S., Khodak, M., and Khandeparkar, H.
\newblock A theoretical analysis of contrastive unsupervised representation learning.
\newblock In \emph{International Conference on Machine Learning}, pp.\  5628--5637. PMLR, 2019.

\bibitem[Shen et~al.(2020)Shen, Jin, Hua, and Huang]{shen2020ttpnet}
Shen, Y., Jin, C., Hua, J., and Huang, D.
\newblock Ttpnet: A neural network for travel time prediction based on tensor decomposition and graph embedding.
\newblock \emph{IEEE Transactions on Knowledge and Data Engineering}, 34\penalty0 (9):\penalty0 4514--4526, 2020.

\bibitem[Shi et~al.(2013)Shi, Zhang, Chen, and Karimi]{shi2013missing}
Shi, F., Zhang, D., Chen, J., and Karimi, H.~R.
\newblock Missing value estimation for microarray data by bayesian principal component analysis and iterative local least squares.
\newblock \emph{Mathematical Problems in Engineering}, 2013\penalty0 (1):\penalty0 162938, 2013.

\bibitem[Simini et~al.(2021)Simini, Barlacchi, Luca, and Pappalardo]{simini2021deep}
Simini, F., Barlacchi, G., Luca, M., and Pappalardo, L.
\newblock A deep gravity model for mobility flows generation.
\newblock \emph{Nature communications}, 12\penalty0 (1):\penalty0 6576, 2021.

\bibitem[Song et~al.(2020)Song, Meng, and Ermon]{song2020denoising}
Song, J., Meng, C., and Ermon, S.
\newblock Denoising diffusion implicit models.
\newblock \emph{arXiv preprint arXiv:2010.02502}, 2020.

\bibitem[Song et~al.(2018)Song, Wang, Xiao, Han, Cai, and Shi]{song2018anomalous}
Song, L., Wang, R., Xiao, D., Han, X., Cai, Y., and Shi, C.
\newblock Anomalous trajectory detection using recurrent neural network.
\newblock In \emph{Advanced Data Mining and Applications: 14th International Conference, ADMA 2018, Nanjing, China, November 16--18, 2018, Proceedings 14}, pp.\  263--277. Springer, 2018.

\bibitem[Song et~al.(2017)Song, Ouyang, Du, Wang, and Xiong]{song2017recovering}
Song, X., Ouyang, Y., Du, B., Wang, J., and Xiong, Z.
\newblock Recovering individual’s commute routes based on mobile phone data.
\newblock \emph{Mobile Information Systems}, 2017\penalty0 (1):\penalty0 7653706, 2017.

\bibitem[Tan et~al.(2013)Tan, Feng, Feng, Wang, Zhang, and Li]{tan2013tensor}
Tan, H., Feng, G., Feng, J., Wang, W., Zhang, Y.-J., and Li, F.
\newblock A tensor-based method for missing traffic data completion.
\newblock \emph{Transportation Research Part C: Emerging Technologies}, 28:\penalty0 15--27, 2013.

\bibitem[Wang et~al.(2016)Wang, Gu, Wu, Liu, and Xiong]{wang2016traffic}
Wang, J., Gu, Q., Wu, J., Liu, G., and Xiong, Z.
\newblock Traffic speed prediction and congestion source exploration: A deep learning method.
\newblock In \emph{2016 IEEE 16th international conference on data mining (ICDM)}, pp.\  499--508. IEEE, 2016.

\bibitem[Wang et~al.(2019)Wang, Wu, Lu, Zhao, and Feng]{wang2019deep}
Wang, J., Wu, N., Lu, X., Zhao, W.~X., and Feng, K.
\newblock Deep trajectory recovery with fine-grained calibration using kalman filter.
\newblock \emph{IEEE Transactions on Knowledge and Data Engineering}, 33\penalty0 (3):\penalty0 921--934, 2019.

\bibitem[Wang et~al.(2021)Wang, Wu, and Zhao]{wang2021personalized}
Wang, J., Wu, N., and Zhao, W.~X.
\newblock Personalized route recommendation with neural network enhanced search algorithm.
\newblock \emph{IEEE Transactions on Knowledge and Data Engineering}, 34\penalty0 (12):\penalty0 5910--5924, 2021.

\bibitem[Wang et~al.(2018)Wang, Fu, and Ye]{wang2018learning}
Wang, Z., Fu, K., and Ye, J.
\newblock Learning to estimate the travel time.
\newblock In \emph{Proceedings of the 24th ACM SIGKDD international conference on knowledge discovery \& data mining}, pp.\  858--866, 2018.

\bibitem[Wei et~al.(2024)Wei, Lin, Lin, Guo, Zhang, and Wan]{wei2024micro}
Wei, T., Lin, Y., Lin, Y., Guo, S., Zhang, L., and Wan, H.
\newblock Micro-macro spatial-temporal graph-based encoder-decoder for map-constrained trajectory recovery.
\newblock \emph{IEEE Transactions on Knowledge and Data Engineering}, 2024.

\bibitem[Wu et~al.(2023)Wu, Hu, Liu, Zhou, Wang, and Long]{wu2023timesnet}
Wu, H., Hu, T., Liu, Y., Zhou, H., Wang, J., and Long, M.
\newblock Timesnet: Temporal 2d-variation modeling for general time series analysis.
\newblock In \emph{International Conference on Learning Representations}, 2023.

\bibitem[Wu et~al.(2019)Wu, Wang, Zhao, and Jin]{wu2019learning}
Wu, N., Wang, J., Zhao, W.~X., and Jin, Y.
\newblock Learning to effectively estimate the travel time for fastest route recommendation.
\newblock In \emph{Proceedings of the 28th ACM International Conference on Information and Knowledge Management}, pp.\  1923--1932, 2019.

\bibitem[Wu et~al.(2020)Wu, Zhao, Wang, and Pan]{wu2020learning}
Wu, N., Zhao, X.~W., Wang, J., and Pan, D.
\newblock Learning effective road network representation with hierarchical graph neural networks.
\newblock In \emph{Proceedings of the 26th ACM SIGKDD international conference on knowledge discovery \& data mining}, pp.\  6--14, 2020.

\bibitem[Xia et~al.(2021)Xia, Qi, Feng, Xu, Sun, Guo, and Li]{xia2021attnmove}
Xia, T., Qi, Y., Feng, J., Xu, F., Sun, F., Guo, D., and Li, Y.
\newblock Attnmove: History enhanced trajectory recovery via attentional network.
\newblock In \emph{Proceedings of the AAAI Conference on Artificial Intelligence}, volume~35, pp.\  4494--4502, 2021.

\bibitem[Xu et~al.(2023)Xu, Bazarjani, Chi, Choi, and Fu]{gc-vrnn}
Xu, Y., Bazarjani, A., Chi, H.-g., Choi, C., and Fu, Y.
\newblock Uncovering the missing pattern: Unified framework towards trajectory imputation and prediction.
\newblock In \emph{Proceedings of the IEEE/CVF Conference on Computer Vision and Pattern Recognition}, pp.\  9632--9643, 2023.

\bibitem[Yang et~al.(2017)Yang, Sun, Zhao, Liu, and Chang]{yang2017neural}
Yang, C., Sun, M., Zhao, W.~X., Liu, Z., and Chang, E.~Y.
\newblock A neural network approach to jointly modeling social networks and mobile trajectories.
\newblock \emph{ACM Transactions on Information Systems (TOIS)}, 35\penalty0 (4):\penalty0 1--28, 2017.

\bibitem[Yang et~al.(2014)Yang, Zhang, Zheng, and Yu]{yang2014modeling}
Yang, D., Zhang, D., Zheng, V.~W., and Yu, Z.
\newblock Modeling user activity preference by leveraging user spatial temporal characteristics in lbsns.
\newblock \emph{IEEE Transactions on Systems, Man, and Cybernetics: Systems}, 45\penalty0 (1):\penalty0 129--142, 2014.

\bibitem[Yao et~al.(2017)Yao, Zhang, Huang, and Bi]{yao2017serm}
Yao, D., Zhang, C., Huang, J., and Bi, J.
\newblock Serm: A recurrent model for next location prediction in semantic trajectories.
\newblock In \emph{Proceedings of the 2017 ACM on Conference on Information and Knowledge Management}, pp.\  2411--2414, 2017.

\bibitem[Yuan \& Qiao(2024)Yuan and Qiao]{yuan2024diffusion}
Yuan, X. and Qiao, Y.
\newblock Diffusion-ts: Interpretable diffusion for general time series generation.
\newblock In \emph{International Conference on Learning Representations}, 2024.

\bibitem[Zandbergen(2014)]{zandbergen2014ensuring}
Zandbergen, P.~A.
\newblock Ensuring confidentiality of geocoded health data: Assessing geographic masking strategies for individual-level data.
\newblock \emph{Advances in medicine}, 2014\penalty0 (1):\penalty0 567049, 2014.

\bibitem[Zhang et~al.(2023)Zhang, Tan, Peng, Xu, Wang, Lu, Wu, Sai, Cai, Kummer, et~al.]{zhang2023heterogeneous}
Zhang, J., Tan, S., Peng, C., Xu, X., Wang, M., Lu, W., Wu, Y., Sai, B., Cai, M., Kummer, A.~G., et~al.
\newblock Heterogeneous changes in mobility in response to the sars-cov-2 omicron ba. 2 outbreak in shanghai.
\newblock \emph{Proceedings of the National Academy of Sciences}, 120\penalty0 (42):\penalty0 e2306710120, 2023.

\bibitem[Zheng et~al.(2014)Zheng, Capra, Wolfson, and Yang]{zheng2014urban}
Zheng, Y., Capra, L., Wolfson, O., and Yang, H.
\newblock Urban computing: concepts, methodologies, and applications.
\newblock \emph{ACM Transactions on Intelligent Systems and Technology (TIST)}, 5\penalty0 (3):\penalty0 1--55, 2014.

\bibitem[Zhu et~al.(2024{\natexlab{a}})Zhu, Ye, Zhang, Zhao, and Yu]{10.5555/3666122.3668965}
Zhu, Y., Ye, Y., Zhang, S., Zhao, X., and Yu, J.~J.
\newblock Difftraj: generating gps trajectory with diffusion probabilistic model.
\newblock In \emph{Proceedings of the 37th International Conference on Neural Information Processing Systems}, 2024{\natexlab{a}}.

\bibitem[Zhu et~al.(2024{\natexlab{b}})Zhu, Yu, Zhao, Liu, Ye, Chen, Zhang, Wei, and Liang]{zhu2024controltraj}
Zhu, Y., Yu, J.~J., Zhao, X., Liu, Q., Ye, Y., Chen, W., Zhang, Z., Wei, X., and Liang, Y.
\newblock Controltraj: Controllable trajectory generation with topology-constrained diffusion model.
\newblock In \emph{Proceedings of the 30th ACM SIGKDD Conference on Knowledge Discovery and Data Mining}, pp.\  4676--4687, 2024{\natexlab{b}}.

\end{thebibliography}
\bibliographystyle{icml2025}

\newpage
\appendix
\onecolumn

\section{Detailed Related Work} \label{sec:appendix:relw}
\textbf{Spatial-Temporal Sequence Imputation.} 
Imputation methods for spatial-temporal sequences can be broadly divided into two categories: traditional methods and deep learning methods. Traditional methods include linear interpolation \cite{blu2004linear} and mean value filling, which are fast but overly simplistic and struggle to estimate the overall data distribution \cite{huang2023robust}. More advanced probabilistic methods, such as probabilistic PCA \cite{qu2009ppca} and expectation maximization \cite{shi2013missing}, aim to capture the data distribution more accurately. Autoregressive methods like Vector Autoregressive (VAR) \cite{lutkepohl2013vector} and matrix/tensor-based methods (e.g., Tucker decomposition \cite{tan2013tensor}) have also been used to address missing data in spatial-temporal contexts.

Deep learning-based imputation methods can be further divided into non-generative and generative approaches. Non-generative methods primarily rely on RNNs and attention mechanisms. For example, GRU-D \cite{che2018recurrent} proposes a variant of the gated recurrent unit (GRU) to handle missing data in time series, while TimesNet \cite{wu2023timesnet} leverages 2D convolutional neural networks to model temporal dependencies. Attention-based methods, such as CDSA \cite{ma2019cdsa} and SAITS \cite{du2023saits}, focus on capturing both short-term and long-term dependencies across multiple dimensions (time, location, measurement). Generative methods include variational autoencoders (VAEs) \cite{doersch2016tutorial}, generative adversarial networks (GANs) \cite{goodfellow2020generative}, and diffusion probabilistic models, which have become increasingly popular. For instance, Diffusion-TS \cite{yuan2024diffusion} combines diffusion models with time series decomposition to address missing data.

\textbf{Trajectory Data Mining.}
Trajectory data mining based on deep learning methods can be categorized into several tasks, including trajectory forecasting, travel time estimation, and anomaly detection \cite{chen2024deep}. Trajectory forecasting involves predicting future locations\cite{wang2021personalized,wu2019learning} or traffic conditions\cite{wu2020learning,liu2024full,ji2023spatio,jiang2023pdformer,ji2022stden}. Common approaches include CNN-based models \cite{bao2021bilstm} and RNN-based models \cite{yang2017neural, yao2017serm}, with recent advances exploring diffusion techniques like BCDiff \cite{bcdiff} that bidirectionally refine historical and future trajectories through coupled diffusion models with adaptive gating mechanisms. Travel Time Estimation (TTE) or Estimated Time of Arrival (ETA) involves analyzing trajectory sequences to predict travel time. For example, eRCNN \cite{wang2016traffic} uses raw GPS data with a recurrent convolutional neural network to estimate travel time and speed. Road-based TTE approaches, such as WDR \cite{wang2018learning}, model the correlation between trips and roads using a regression framework.

Trajectory anomaly detection aims to identify abnormal movement patterns. Offline detection methods, like ATD-RNN \cite{song2018anomalous}, use RNNs with fully connected layers for anomaly detection. Online detection methods leverage reinforcement learning to model the transition probability between road segments, treating anomaly detection as a sequential decision problem \cite{chen2024deep}. GM-VSAE \cite{liu2020online} adapts an RNN-based VAE model to learn the probability distribution in the latent space.

Recent advancements address data incompleteness challenges through unified frameworks, GC-VRNN \cite{gc-vrnn} pioneers joint trajectory imputation and prediction using multi-space graph neural networks to capture spatio-temporal missing patterns and temporal decay modules for information recovery. Mobility generation tasks have also gained attention, with models like DiffTraj \cite{10.5555/3666122.3668965} utilizing diffusion models to generate synthetic trajectories at the population level. Our proposed task combines trajectory forecasting and mobility generation, where generative trajectory interpolation aligns with mobility generation, and generalized trajectory interpolation is considered a higher-order forecasting task.

\textbf{Mobility Data Synthesizing.}
The generation of synthetic mobility data has been extensively studied to address privacy concerns, data scarcity, and high collection costs \cite{jia2020population, zhang2023heterogeneous, zheng2014urban}. Early non-generative approaches primarily relied on statistical models \cite{simini2021deep,wang2019deep}, perturbation techniques \cite{zandbergen2014ensuring}, or simulations \cite{simini2021deep}. While these methods offer insights into movement dynamics, they often fail to capture complex spatial-temporal relationships in real-world scenarios \cite{Pappalardo_2023}.

With advancements in deep learning, generative approaches have gained prominence. Variational Autoencoders (VAEs) like TrajVAE \cite{chen2021trajvae} leverage temporal dependencies to produce realistic trajectories, while GAN-based frameworks such as TS-TrajGen \cite{jiang2023continuous} use coarse-to-fine modeling to generate synthetic trajectories from spatial grid transformations. However, these models often face limitations in achieving high-resolution fidelity, particularly when translating grid-based representations into fine-grained data.

Graph-based approaches have been widely investigated due to their ability to capture spatial-temporal relationships effectively. For example, RNTrajRec \cite{chen2023rntrajrec} employs a graph-based framework that integrates graph representations of trajectory points and spatial-temporal transformers to model dependencies along the trajectory, significantly enhancing trajectory recovery accuracy. Similarly, MM-STGED \cite{wei2024micro} utilizes a graph-based encoder-decoder structure to represent trajectories as spatial-temporal graphs, capturing both micro-level semantics of GPS points and macro-level semantics of shared travel patterns.

Attention-based architectures, such as AttnMove \cite{xia2021attnmove}, leverage attention mechanisms to model spatial-temporal correlations explicitly, facilitating the reconstruction of missing trajectory data and improving performance in downstream applications. Recent innovations in trajectory generation include the use of denoising diffusion probabilistic models (DDPMs) \cite{ho2020denoising}, which iteratively refine noisy inputs to produce high-fidelity synthetic data. For instance, DiffTraj \cite{10.5555/3666122.3668965} captures spatial-temporal dependencies without relying on intermediate transformations, offering significant advantages in privacy preservation and data utility. Additionally, ControlTraj \cite{zhu2024controltraj} extends the diffusion framework by integrating conditional signals for controllable generation, improving its applicability across varied scenarios.
\section{Detailed Denoising Network}

\subsection{Denoising Diffusion Probabilistic Model}
The diffusion probabilistic model has gained increasing attention in recent years for its success in various data generation tasks. The model consists of a \textit{forward process} that gradually perturbs the data distribution with noise, and a \textit{reverse (denoising) process} that learns to reconstrust the original data distribution.

\textbf{Forward Process.} Given a set of data samples $\boldsymbol{x}_0 \sim q(\boldsymbol{x}_0)$, the forward process adds $T$ time-steps of Gaussian noise $\mathcal{N}(\cdot)$ to it, where $T$ is an adjustable parameter. Formally, the forward process can be defined as a Markov chain from data $\boldsymbol{x}_0$ to the latent variable $\boldsymbol{x}_T$:
\begin{align}
    q(\boldsymbol{x}_{1:T} \mid \boldsymbol{x}_0) &= \prod_{t = 1}^T q(\boldsymbol{x}_{t} \mid \boldsymbol{x}_{t-1}) \\
    q(\boldsymbol{x}_{t} \mid \boldsymbol{x}_{t-1}) &= \mathcal{N}(\boldsymbol{x}_{t}; \sqrt{1- \beta_t}\boldsymbol{x}_{t-1}, \beta_t\boldsymbol{I}),
\end{align}
in which $\{\beta_t \in (0,1)\}_{t = 1}^T (\beta_1 < \beta_2 < ... < \beta_T)$ is the corresponding variance schedule. Since it is impractical to back-propagate the gradient by sampling from a Gaussian distribution, we adopt a parameterization trick to keep the gradient derivable and the $\boldsymbol{x}_t$ can be expressed as $\boldsymbol{x}_t = \sqrt{\bar{\alpha}_t} \boldsymbol{x}_0 + \sqrt{1 - \bar{\alpha}_t} \boldsymbol{\epsilon}$, where $\boldsymbol{\epsilon} \sim \mathcal{N}(0, \boldsymbol{I})$ and $\bar{\alpha}_t = \prod_{i = 1}^t(1 - \beta_i)$.

\textbf{Reverse Process.} The reverse diffusion process, also known as the denoising processing, aims to reconstruct the original data distribution from the noisy data $\boldsymbol{x}_T \sim \mathcal{N}(0, \boldsymbol{I})$. Accordingly, this process can be formulated by the following Markov chain:
\begin{align}
    p_\theta(\boldsymbol{x}_{0:T} ) &= p(\boldsymbol{x}_T)\prod_{t = 1}^T p_\theta(\boldsymbol{x}_{t-1} \mid \boldsymbol{x}_{t}) \\
    p_\theta(\boldsymbol{x}_{t-1} \mid \boldsymbol{x}_{t}) &= \mathcal{N}\left(\boldsymbol{x}_{t-1}; \boldsymbol{\mu}_\theta(\boldsymbol{x}_t, t), \boldsymbol{\sigma}_\theta(\boldsymbol{x}_t, t)^2\boldsymbol{I}\right), \label{inverse}
\end{align}
where $\boldsymbol{\mu}_\theta(\boldsymbol{x}_t, t)$ and $\boldsymbol{\sigma}_\theta(\boldsymbol{x}_t, t)$ are the mean and variance parameterized by $\theta$, respectively. Based on the literature, for any $\tilde{\beta}_t = \frac{1 - \bar{\alpha}_{t-1}}{1 - \bar{\alpha}_t}\beta_t(t>1)$ and $\tilde{\beta}_1 = \beta_1$, the parameterizations of $\boldsymbol{\mu}_\theta$ and $\boldsymbol{\sigma}_\theta$ are defined by:
\begin{align}
    \boldsymbol{\mu}_\theta(\boldsymbol{x}_t, t) = \frac{1}{\sqrt{\alpha}_t}\left(\boldsymbol{x}_t - \frac{\beta_t}{\sqrt{1 - \bar{\alpha}_t}}\boldsymbol{\epsilon}(\boldsymbol{x}_t, t)\right), \boldsymbol{\sigma}_\theta(\boldsymbol{x}_t, t) = \tilde{\beta}_t^\frac{1}{2}.
    \label{mu}
\end{align}

\newpage
\section{Method} \label{sec:appendix:method}

\let\oldnl\nl
\newcommand{\nonl}{\renewcommand{\nl}{\let\nl\oldnl}}
\begin{algorithm}[htbp]
  \caption{Training of ProDiff}
  \raggedright

\begin{algorithmic} [ht]
    \FOR{ $i = 1, 2, \ldots,$}
      \STATE Get \textcolor{black}{base condition $\boldsymbol{B}^{c}$}
      \STATE Get \textcolor{purple}{prototype condition $\boldsymbol{P}^{c}$} by \textcolor{purple}{PCE network}
      \STATE Get \textcolor{purple}{$WD(\boldsymbol{B}^{c}), WD(\boldsymbol{P}^{c})$} by \textcolor{purple}{Wide \& Deep network}
      \STATE Get \textcolor{purple}{conditional guidance $\mathcal{J}^{c} =WD(\boldsymbol{B}^{c}) + WD(\boldsymbol{P}^{c})$}
      \STATE $f_\gamma(\mathbf{Z}_0) = \mathcal{J}^c$
      \STATE Sample $\mathbf{Z} \sim p$ where $p$ represents the distribution of original data 
      \STATE Sample $t \sim \mathcal{U}[0, T]$, $\epsilon \sim \mathcal{N}(0, \mathbf{I}_{l \times d})$
      \STATE $\mathbf{Z}_t = \sqrt{\bar{\alpha}_t}\mathbf{Z}_0 +\sqrt{1-\bar{\alpha}_t} \boldsymbol{\epsilon}$
      \STATE Updating the gradient $\nabla_{\theta / \gamma}\mathcal{L}_{J}$ which means optimizing $\mathbb{E}_{t \sim \mathcal{U}[0, T]}\mathbb{E}_{\mathbf{Z}_0 \sim p, \mathbf{\epsilon} \sim \mathcal{N}}\left[\nabla_{\theta / \gamma}\lVert \mathbf{\epsilon} - \mathbf{\epsilon}_{\theta}(\mathbf{Z}_t, t, f_\gamma(\mathbf{Z}_0))\rVert^2\right]$
    \ENDFOR
    \label{trainalg}
\end{algorithmic}
\end{algorithm}

\begin{algorithm}[htbp]
  \caption{Sampling of ProDiff}
  \raggedright

\begin{algorithmic} [1]
    \STATE Get data and Sample $\tilde{\mathbf{Z}}_T \sim \mathcal{N}(0,\mathbf{I})$
      \STATE Get \textcolor{black}{base condition $\boldsymbol{B}^{c}$}
      \STATE Get \textcolor{purple}{prototype condition $\boldsymbol{P}^{c}$} by \textcolor{purple}{PCE network}
      \STATE Get \textcolor{purple}{$WD(\boldsymbol{B}^{c}), WD(\boldsymbol{P}^{c})$} by \textcolor{purple}{Wide \& Deep network}
      \STATE Get \textcolor{purple}{conditional guidance $\mathcal{J}^{c} =WD(\boldsymbol{B}^{c}) + WD(\boldsymbol{P}^{c})$}
    \FOR{ $t = T, T-S, \ldots, 1$}
        \STATE Compute $\mu_{\theta}\left(\tilde{\mathbf{Z}}_t, t, \mathcal{J}^{c}\right)$ according to Eq.\eqref{mu}
      \STATE Compute $p_{\theta}\left(\tilde{\mathbf{Z}}_{t-1} \mid \tilde{\mathbf{Z}}_t, \mathcal{J}^{c}\right)$ according to Eq.\eqref{inverse}
    \ENDFOR \\
    \STATE \textbf{Return:} $\tilde{\mathbf{Z}}_0$

    \label{inferalg}
\end{algorithmic}

\end{algorithm}

\section{Experiment} \label{sec:appendix:experiment}

\subsection{Dataset}
We evaluate the performance of ProDiff and all baselines methods on two datasets, \textbf{WuXi} and \textbf{FourSquare}. 

The mobile phone dataset WuXi used in this study were collected between October 24, 2013, and March 24, 2014, in Wuxi, China, encompassing approximately six million users evenly distributed across the area. Every hour, these users generate around 40 million raw records, each containing essential location information, including cell-id and area-id, which correspond to specific cell towers. 
Each record in the dataset includes four key components: user ID, cell tower ID, timestamp, and a tag. The timestamp indicates the exact moment the record was created, while the tag specifies the type of activity associated with the record. 
For the purpose of this study, we focused on data from ten consecutive days, concatenating individual trajectories during this period. This subset includes 33000 active users and 671,124 location updates, of which 30,000 users are used for training and 3,000 users for testing.

The FourSquare dataset contains Foursquare check-ins over ten months (from April 12, 2012, to February 16, 2013), filtered for noise and invalid check-ins. It includes active users in two major cities, New York and Tokyo, with each check-in associated with a timestamp, GPS coordinates, and semantic meaning. We did not use the taxi-related dataset because human trajectories have a higher degree of freedom compared to car trajectories.
Due to the volume of data, only Tokyo data was used on the FourSquare dataset. Tab. \ref{tab:hm} summarizes the statistical information of these two datasets, which includes 2,293 active users with 573,703 location updates.

\begin{table}[htbp]
  \centering
  \caption{Statistics of two human mobility datasets.}
    \begin{tabular}{lcc}
    \toprule
    \multicolumn{1}{c}{Dataset} & WuXi & FourSquare\\
    \midrule
    Time Span (day)  & 111 & 310   \\
    Used Time Span (day) & 10  &  310 \\
    Train Active Users & 30000  &  1834 \\
    Test Active Users & 3000  &  459 \\
    Location Updates & 671,124  &  573,703  \\
    Average Distance (meter) & 3336.33 & 4301.51 \\
    Average Time (hour) & 7.8 & 37.15 \\
    \bottomrule
    \end{tabular}%
  \label{tab:hm}%
\end{table}%

\subsection{Preprocess}
In trajectory data analysis, careful preparation of raw data is fundamental to ensure the reliability and precision of computational models. Our preprocessing approach transforms raw GPS coordinates into a format optimized for training deep learning systems, focusing on two key steps: segmentation and normalization.

\textbf{Segmentation} involves dividing continuous trajectory data into fixed-length segments using a sliding window method. This technique incrementally generates samples from a single trajectory. Trajectories matching the target length are directly included as individual samples, while longer paths are systematically partitioned into uniform segments. This creates discrete, standardized inputs for model training.

\textbf{Normalization} adapts the data for diffusion-based models, which rely on introducing and removing Gaussian noise during training. To align with the noise distribution, spatial coordinates are scaled to a dimensionless, standardized range (e.g., $[0,1]$). This eliminates scale variations between features, allowing the model to focus on spatial patterns rather than magnitude differences. Crucially, the process is fully reversible—after model inference, outputs can be rescaled to their original geographic coordinates, preserving real-world interpretability.

\subsection{Evaluation Metric and Baseline}
The trajectory imputation task aims to fill in missing points as accurately as possible under the given point conditions. To assess performance, we propose a novel trajectory coverage metric, measuring the percentage of generated locations within a specified distance from the groundtruth. Given a threshold $\tau$, we count the number of generated points within $\tau$ distance from the groundtruth and divide by the total length of the trajectory to limit it to the $[0, 1]$ interval. For any trajectory $\mathbf{S}_i = [\mathbf{s}_{i, 1}, \mathbf{s}_{i, 2}, ... , \mathbf{s}_{i, k}]$ of length $k$, we construct a masked trajectory $\mathbf{S}_i^m = [\mathbf{s}_{i, 1},  \mathbf{s}_{i, k}]$. Given this condition and threshold $\tau$, ProDiff generates the missing points, and we calculate the trajectory coverage $TC@\tau$ as the following equation:
\begin{align}
    TC@\tau = \frac{1}{k}\sum_{j=1}^{k} \mathbb{I}(d(\hat{\mathbf{s}_{i, j}}, \mathbf{s}_{i, j}) < \tau),
\end{align}
where $\hat{\mathbf{s}_{i, j}}$ is the generated points from ProDiff, and $\mathbb{I}(\cdot)$ is an indicator function that equals 1 when $d(\hat{\mathbf{s}_{i, j}}, \mathbf{s}_{i, j})$ is less than the threshold $\tau$ and $0$ otherwise. We use the haversine function as the distance metric, which calculates the great-circle distance between two points on the Earth's surface, accurately reflecting the true distance by converting latitude and longitude into radians.

We select some traditional methods and the current state-of-the-art spatio-temporal sequence methods based on the Diffusion model, as well as trajectory-related methods to realize the trajectory imputation task. The compared baselines are set as follows:

\textbf{VAR}: 
Vector Autoregression (VAR) is a traditional model used to capture the linear interdependencies among multiple time series data. By considering each variable's own lagged values and the lagged values of other variables in the system, VAR models can effectively analyze the dynamic relationships and forecast future movements of the variables \cite{lutkepohl2005new}.

\textbf{SAITS}:
SAITS is an advanced model\cite{du2023saits} designed to handling missing data in time series analysis. By leveraging self-attention mechanisms, SAITS aims to capture both short-term and long-term dependencies within time series data, This model stands out due to its ability to focus on the most relevant parts of the input data. 

\textbf{TimesNet}: TimesNet is a progressive model \cite{wu2023timesnet} which treats time series data as 2D tensors, allowing it to leverage powerful 2D convolutional neural networks to model temporal dependencies. This approach is suitable for a wide range of applications, including missing value handling, forecasting, and anomaly detection.

\textbf{Diffusion-TS}: A current SOTA method \cite{yuan2024diffusion} for time series generation task based on diffusion model and it also applies to missing value processing tasks. Diffusion-TS decomposes time series into interpretable variables, combining seasonal trend decomposition techniques and denoising diffusion models.

\textbf{DiffTraj}: Generating GPS Trajectory with Diffusion Probabilistic Model is a SOTA model\cite{10.5555/3666122.3668965} designed for generating realistic GPS trajectories. DiffTraj progressively refines random noise into coherent and plausible GPS trajectory data through a series of probabilistic steps. It is worth noting that while DiffTraj can also be used for the trajectory imputation task, it generates trajectories at the population level, which is fundamentally different from what we have done at the individual level.

We further evaluated our model against external baselines and metrics. Specifically, we incorporated RNTrajRec\cite{chen2023rntrajrec}, TS-TrajGen\cite{jiang2023continuous}, MM-STGED\cite{wei2024micro}, and AttnMove\cite{xia2021attnmove} as baseline models:

\textbf{RNTrajRec}: RNTrajRec\cite{chen2023rntrajrec} employs a graph-based framework that integrates graph representations of trajectory points and spatial-temporal transformers to model dependencies along the trajectory, significantly enhancing trajectory recovery accuracy.

\textbf{TS-TrajGen}: TS-TrajGen\cite{jiang2023continuous} proposes a hierarchical generation framework that employs coarse-to-fine modeling to synthesize realistic trajectories. It first learns spatial grid-based latent representations to capture macroscopic movement patterns, then refines trajectories through adaptive spatial transformations and temporal interpolation. This approach effectively addresses the sparsity of raw GPS data while preserving topological consistency with road networks.

\textbf{MM-STGED}: MM-STGED\cite{wei2024micro} utilizes a graph-based encoder-decoder structure to represent trajectories as spatial-temporal graphs, capturing micro-level semantics of GPS points and macro-level semantics of shared travel patterns.

\textbf{AttnMove}: AttnMove\cite{xia2021attnmove}, leverage attention mechanisms to model spatial-temporal correlations explicitly, facilitating the reconstruction of missing trajectory data and improving performance in downstream applications.

To ensure fair comparisons, specific modules in these baselines were removed to avoid reliance on unavailable additional information from our datasets. Additionally, we introduced extra metrics to assess the spatial distribution of generated trajectories:
\begin{itemize}
    \item Density: Measures the cosine similarity of grid density between real and generated trajectories (higher is better).
    \item Distance: Evaluates the difference in travel distance between real and generated data, calculated as the sum of distances between consecutive points (lower is better).
    \item Segment Distance: Assesses the difference in segment distance between real and generated data, defined as the distance between consecutive points (lower is better).
    \item Radius: Evaluates the root mean square distance of all activity locations from the central location, indicating the spatial range (lower is better).
    \item MAE: Mean absolute error, measuring the average magnitude of errors between real and generated trajectories (lower is better).
    \item RMSE: Root mean square error, evaluating the square root of the average squared differences between predicted and actual values (lower is better).
\end{itemize}

\subsection{Exploratory Study}
To assess ProDiff's performance under varying information conditions, we conducted two additional sets of experiments for the trajectory imputation task. Specifically, given a trajectory length of 10, in addition to fixing the begin and end points, supplementary information points were added to guide the model to accomplish better generation, as shown in the first four rows of Tab. \ref{tab:info}. Meanwhile, motivated by the conclusion in the literature \cite{de2013unique} (that 95\% of the personnel's trajectories can be determined by arbitrarily giving 4 points), we modify the information of the fixed points to randomly selecting x points during both training and testing. This scenario, represented in the last four rows of Tab. \ref{tab:info}, is more challenging since the selected points vary and thus introduce more complexity. From the results of whole table, some key findings can be obtained: (i) Impact of Increased Information: The results show a significant improvement in trajectory imputation when the number of given points increases from 2 to 4. Beyond four points, the enhancement in performance becomes marginal. (ii) Fixed vs. Random Points: When fixing four points, our method's performance aligns closely with the findings of \cite{de2013unique}. However, when points are selected randomly, the model's performance diverges more noticeably from the literature's conclusion. This discrepancy likely arises because random points disrupt the consistency of trajectory sampling, increasing the difficulty of model learning. The results highlight the potential of ProDiff in accurately imputing missing trajectory points with a sufficient number of fixed reference points. However, the challenge remains when dealing with randomly selected points, indicating an area for future improvement.

\begin{table}[htbp]
  \centering
  \caption{Impact of different diffusion steps (d).}
    \begin{tabular}{lrrrrr}
    \toprule
    \multicolumn{1}{c}{d} & \multicolumn{1}{c}{TC@2k} & \multicolumn{1}{c}{TC@4k} & \multicolumn{1}{c}{TC@6k} & \multicolumn{1}{c}{TC@8k} & \multicolumn{1}{c}{TC@10k} \\
    \midrule
    100  & 0.5750      &  0.7445     & 0.8329      &  0.8853     &  0.9187 \\
    300 & 0.6015    &  0.7697     &  0.8524     & 0.9005      & 0.9300  \\
    500  & 0.5978   & 0.7686   &  0.8515  &  0.8992  & 0.9285  \\
    700 & 0.5881   &  0.7551     & 0.8399      & 0.8897      & 0.9216 \\
    \bottomrule
    \end{tabular}%
  \label{tab:d}%
\end{table}%

\begin{table}[htbp]
  \centering
  \caption{Comparison of the performance of fixed and randomized trajectory points with different amount of information. "3/10" means that given a trajectory of length 10, three points are provided at indices [0, 4, 9]. Similarly, "4/10" provides points at indices [0, 3, 6, 9], and "5/10" at indices [0, 2, 4, 6, 9]. $\dagger$ represents the random selecting experiments.}
    \begin{tabular}{lrrrrr}
    \toprule
    \multicolumn{1}{c}{Info} & \multicolumn{1}{c}{TC@2k} & \multicolumn{1}{c}{TC@4k} & \multicolumn{1}{c}{TC@6k} & \multicolumn{1}{c}{TC@8k} & \multicolumn{1}{c}{TC@10k} \\
    \midrule
    2/10  & 0.4996      &  0.6994     & 0.8048      &  0.8667     &  0.9053 \\
    3/10 &  0.5865     &  0.7638     &  0.8498     & 0.8990      & 0.9292 \\
    4/10  & 0.6820     & 0.8305      & 0.8979      & 0.9347      & 0.9561  \\
    5/10 &  0.7362     &  0.8637     & 0.9179      & 0.9466      &  0.9633 \\
    \midrule
    $\text{2/10}^\dagger$  & 0.4143 &  0.6282  &  0.7402  &  0.7996   & 0.8351    \\
    $\text{3/10}^\dagger$ & 0.5364    & 0.7098  & 0.7897  &  0.8363   & 0.8663 \\
    $\text{4/10}^\dagger$  & 0.6051      &       0.7501 &  0.8269     &  0.8738    & 0.9047  \\
    $\text{5/10}^\dagger$ & 0.6817    & 0.8062      & 0.8706      & 0.9089      & 0.9336 \\
    \bottomrule
    \end{tabular}%
  \label{tab:info}%
\end{table}%

\subsection{Ablation Study}
Diffusion models typically involve higher computational costs, which may potentially limit their application in large-scale trajectory data scenarios. To tackle the efficiency issue, we have developed two variants that integrate DDIM sampling (ProDDIM) \cite{song2020denoising} and LA (ProDDIM+Linear Attention) \cite{katharopoulos2020transformers}, respectively.

We carried out experiments on the WuXi dataset with k=8. As presented in Tab. \ref{exp: acc}, both variants achieve at least approximately ~10× speed-up compared to ProDDPM, while only experiencing minor performance reductions. These findings indicate that our proposed model serves as a general framework. It can be effectively combined with various acceleration techniques, thereby facilitating its deployment in real-world applications and addressing the computational efficiency concerns associated with diffusion models in the context of large-scale trajectory data.

\begin{table*}[ht]
	\centering
	\caption{Accelerated verisons of the ProDiff model (Thpt: Throughput; PPT: Processing Per Time-unit)}
	\renewcommand{\arraystretch}{1.5}
		\begin{tabular}{lccccccc}
			\hline
			\multicolumn{1}{l}{Method} & \multicolumn{1}{c}{TC@2k} & \multicolumn{1}{c}{TC@4k} & \multicolumn{1}{c}{TC@6k} & \multicolumn{1}{c}{TC@8k} & \multicolumn{1}{c}{TC@10k} & \multicolumn{1}{c}{Thpt(s/sample)} & \multicolumn{1}{c}{PPT(sample/s)} \\ \hline
			ProDDPM (Ours)                      & 0.5752                             & 0.7501                             & 0.8236                             & 0.8663                             & 0.8945                              & 77.9346                                     & 0.0128                                     \\
			ProDDIM(Ours)                       & 0.5430                             & 0.7131                             & 0.7773                             & 0.8303                             & 0.8741                              & 788.6852                                    & 0.0013                                     \\
			ProDDIM+LA(Ours)                    & 0.5350                             & 0.7197                             & 0.7725                             & 0.836                              & 0.8785                              & 768.5845                                    & 0.0013                                     \\ \hline
		\end{tabular}
    \label{exp: acc}
\end{table*}

\end{document}